\documentclass{article}

\usepackage[utf8]{inputenc}
\usepackage{amsmath,amscd,amssymb,amsthm}

\newtheorem{definition}{Definition}
\newtheorem{lemma}[definition]{Lemma}
\newtheorem{theorem}[definition]{Theorem}

\usepackage{verbatim}
\usepackage{thmtools}
\usepackage{thm-restate}
\usepackage{tikz}
\usetikzlibrary{shapes,arrows}
\usetikzlibrary{positioning}
\usepackage{fullpage}

\tikzstyle{reduction} = [rectangle, draw, fill=blue!20,  align=center,
    text centered, rounded corners, minimum height=4em, minimum width=4em]
\tikzstyle{reductiongreen} = [rectangle, draw=red!50!green!100, fill=red!40!green!40, align=center,
    text centered, rounded corners, minimum height=4em, minimum width=4em]
\tikzstyle{reductionpurple} = [rectangle, draw=purple, fill=purple!40, align=center,
    text centered, rounded corners, minimum height=4em, minimum width=4em]
\tikzstyle{reductionblue} = [rectangle, draw=blue, fill=blue!20, align=center,
    text centered, rounded corners, minimum height=4em, minimum width=4em]
\tikzstyle{application} = [rectangle, draw, fill=red!20, align=center,
    text centered, minimum height=4em, minimum width=4em, rounded corners]
\tikzstyle{application2} = [rectangle, draw, fill=black!20, align=center,
    text centered, minimum height=4em, minimum width=4em, rounded corners]
\tikzstyle{line} = [draw, -latex']
\tikzstyle{setting} = [rectangle, draw, fill=blue!20, align=center,
    text centered, rounded corners, minimum height=4em, minimum width=4em]
\tikzstyle{bound} = [rectangle, draw, fill=red!20, align=center,
    text centered, rounded corners, minimum height=4em, minimum width=4em]
\tikzstyle{thickline} = [draw, -latex', thick]

\declaretheorem[name=Theorem]{Theorem}
\declaretheorem[name=Proposition, numberlike=Theorem]{Proposition}

\declaretheorem[sibling=Theorem]{Definition}

\declaretheorem[sibling=Theorem]{Claim}

\usepackage[framemethod=TikZ]{mdframed}
\mdfdefinestyle{MyFrame}{%
    linecolor=black,
    outerlinewidth=2pt,
    roundcorner=20pt,
    innerrightmargin=20pt,
    innerleftmargin=20pt,
    backgroundcolor=white}
\usepackage{natbib}
\setcitestyle{numbers}
\setcitestyle{square}
\usepackage{hyperref}

\usepackage{caption}

\newcommand{\argmin}{\mathop{\text{argmin}}}

\newcommand{\ol}{\mathcal{A}}
\newcommand{\bol}{\mathcal{A}_{S}}
\newcommand{\onedol}{\mathcal{A_{\text{1D}}}}
\newcommand{\x}{{x_\star}}
\newcommand{\wealth}{\text{Wealth}}
\newcommand{\w}{\mathring{w}}
\renewcommand{\v}{\mathring{v}}

\newcommand{\field}[1]{\mathbb{#1}}

\newcommand{\fX}{\field{X}}

\newcommand{\R}{\field{R}}
\newcommand{\C}{\field{C}}

\newcommand{\E}{\field{E}}

\newcommand{\sign}{{\rm sign}}

\usepackage{graphicx}

\usepackage{algorithm}
\usepackage{algorithmic}

\title{Black-Box Reductions for Parameter-free Online Learning in Banach Spaces}

\author{
  Ashok Cutkosky\\
  Department of Computer Science, Stanford University\\
  \texttt{ashokc@cs.stanford.edu}
  \and
  Francesco Orabona\\
  Department of Computer Science, Stony Brook University\\
  \texttt{francesco@orabona.com}
}

\usepackage{times}
\begin{document}

\maketitle

\begin{abstract}
We introduce several new black-box reductions that significantly improve the design of adaptive and parameter-free online learning algorithms by simplifying analysis, improving regret guarantees, and sometimes even improving runtime. We reduce parameter-free online learning to online exp-concave optimization, we reduce optimization in a Banach space to one-dimensional optimization, and we reduce optimization over a constrained domain to unconstrained optimization. All of our reductions run as fast as online gradient descent. We use our new techniques to improve upon the previously best regret bounds for parameter-free learning, and do so for arbitrary norms.
\end{abstract}

\section{Parameter Free Online Learning}\label{sec:intro}

Online learning is a popular framework for understanding iterative optimization algorithms, including stochastic optimization algorithms or algorithms operating on large data streams. For each of $T$ iterations, an online learning algorithm picks a point $w_t$ in some space $W$, observes a loss function $\ell_t:W\to \R$, and suffers loss $\ell_t(w_t)$. Performance is measured by the \emph{regret}, which is the total loss suffered by the algorithm in comparison to some benchmark point $\w\in W$:
\begin{align*}
R_T(\w) = \sum_{t=1}^T \ell_t(w_t) - \ell_t(\w)~.
\end{align*}
We want to design algorithms that guarantee low regret, even in the face of adversarially chosen $\ell_t$.

To make the problem more tractable, we suppose $W$ is a convex set and each $\ell_t$ is convex (this is called Online Convex Optimization). With this assumption, we can further reduce the problem to online \emph{linear} optimization (OLO) in which each $\ell_t$ must be a linear function. To see the reduction, suppose $g_t$ is a subgradient of $\ell_t$ at $w_t$ ($g_t\in \partial \ell_t(w_t)$). Then $\ell_t(w_t)-\ell_t(\w)\le \langle g_t,w_t-\w\rangle$, which implies $R_T(\w)\le \sum_{t=1}^T \langle g_t,w_t-\w\rangle$. Our algorithms take advantage of this property by accessing $\ell_t$ only through $g_t$ and controlling the linearized regret $\sum_{t=1}^T \langle g_t,w_t-\w\rangle$.

Lower bounds for unconstrained online linear optimization \citep{mcmahan2012no, orabona2013dimension} imply that when $\ell_t$ are $L$-Lipschitz, no algorithm can guarantee regret better than $\Omega(\|\w\| L \sqrt{T\ln(\|\w\|L T+1)})$. Relaxing the $L$-Lipschitz restriction on the losses leads to catastrophically bad lower bounds \citep{cutkosky2017online}, so in this paper we focus on the case where a Lipschitz bound is known, and assume $L=1$ for simplicity.\footnote{One can easily rescale the $g_t$ by $L$ to incorporate arbitrary $L$.}

Our primary contribution is a series of three reductions that simplify the design of \emph{parameter-free} algorithms,\footnote{The name ``parameter-free'' was first used by \citet{chaudhuri2009parameter} for an expert algorithm that does not need to know the entropy of the competitor to achieve the optimal regret bound for any competitor.} that is algorithms whose regret bound is optimal without the need to tune parameters (e.g. learning rates). First, we show that algorithms for online exp-concave optimization imply parameter-free algorithms for OLO (Section~\ref{sec:betfromons}). Second, we show a general reduction from online learning in arbitrary dimensions with any norm to one-dimensional online learning (Section~\ref{sec:1dred}). Finally, given any two convex sets $W\subset V$, we construct an online learning algorithm over $W$ from an online learning algorithm over $V$ (Section~\ref{sec:constrained}).

All of our reductions are very general. We make no assumptions about the inner workings of the base algorithms and are able to consider any norm, so that $W$ may be a subset of a Banach space rather than a Hilbert space or $\R^d$. Each reduction is of independent interest, even for non-parameter-free algorithms, but by combining them we can produce powerful new algorithms. 

First, we use our reductions to design a new parameter-free algorithm that improves upon the prior regret bounds, achieving
\[
R_T(\w) \le \|\w\|\sqrt{\sum_{t=1}^T \|g_t\|_\star^2\ln\left(\|\w\|\sum_{t=1}^T \|g_t\|_\star^2+1\right)},
\]
where $\|\cdot\|$ is any norm and $\|\cdot\|_\star$ is the dual norm ($\|g_t\|_\star=\|g_t\|$ when $\|\cdot\|$ is the $2$-norm). Previous parameter-free algorithms \cite{mcmahan2012no, mcmahan2014unconstrained, orabona2014simultaneous, orabona2016coin, foster2015adaptive, cutkosky2017online, orabona2017training} obtain at best an exponent of $1$ in their dependence on $\|g_t\|_\star$ (which is worse because $\|g_t\|_\star\le 1$ by our 1-Lipschitz assumption). Achieving $\|g_t\|_\star^2$ rather than $\|g_t\|_\star$ can imply asymptotically lower regret when the losses $\ell_t$ are smooth~\citep{srebro2010smoothness}, so this is not merely a cosmetic difference. In addition to the worse regret bound, all prior analyses we are aware of are quite complicated, often involving pages of intricate algebra, and are usually limited to the $2$-norm. In contrast, our techniques are both simpler and more general.

We further demonstrate the power of our reductions through three more applications. In Section \ref{sec:multiscale}, we consider the multi-scale experts problem studied in \citep{foster2017parameter, bubeck2017online} and improve prior regret guarantees and runtimes. In Section \ref{sec:metagrad}, we create an algorithm obtaining $\tilde O(\sqrt{T})$ regret for general convex losses, but logarithmic regret for strongly-convex losses using only first-order information, similar to \citep{vanervan2016metagrad,cutkosky2017stochastic}, but with runtime improved to match gradient descent. Finally, in Section \ref{sec:onsbanachspace} we prove a regret bound of the form $R_T(\w)=\tilde{O}\left(\sqrt{d \sum_{t=1}^T \langle g_t, \w \rangle^2}\right)$ for $d$-dimensional Banach spaces, extending the results of \cite{NIPS2017_7060} to unconstrained domains. We summarize our results in Figure~\ref{fig:diagram}.

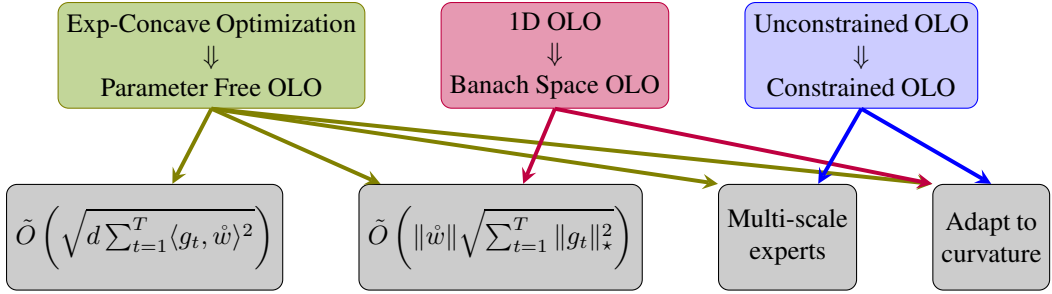
\begin{figure}[t]
\centering
\begin{tikzpicture}[>=stealth,node distance = 1cm and 1cm, auto]
    \node [reductiongreen] (ons) {Exp-Concave Optimization\\$\Downarrow$\\Parameter Free OLO};
    \node [reductionpurple, right = of ons] (banach) {1D OLO\\$\Downarrow$\\Banach Space OLO};
    \node [reductionblue, right = of banach] (unconstrained) {Unconstrained OLO\\$\Downarrow$\\Constrained OLO};

    \node [application2, below left=1cm and -3cm of ons] (app_banach) {$\tilde{O}\left(\sqrt{d \sum_{t=1}^T \langle g_t, \w\rangle^2}\right)$};
    \node [application2, right= of app_banach] (app_paramfree) {$\tilde{O}\left(\|\w\|\sqrt{\sum_{t=1}^T \|g_t\|_\star^2}\right)$};
    \node [application2, right= of app_paramfree] (app_mscale) {Multi-scale\\experts};
    \node [application2, right= of app_mscale] (app_curv) {Adapt to\\curvature};

    \draw[line width=1.5pt] (ons.south) edge[->, color=red!50!green!100] (app_banach);
    \draw[line width=1.5pt] (ons.south) edge[->, color=red!50!green!100] (app_paramfree);
    \draw[line width=1.5pt] (ons.south) edge[->, color=red!50!green!100] (app_mscale.north west);
    \draw[line width=1.5pt] (ons.south) edge[->, color=red!50!green!100] (app_curv.north west);
    
    \draw[line width=1.5pt] (banach.south) edge[->, color=purple] (app_paramfree);
    \draw[line width=1.5pt] (banach.south) edge[->, color=purple] (app_curv.north west);
    
    \draw[line width=1.5pt] (unconstrained.south) edge[->, color=blue] (app_mscale);
    \draw[line width=1.5pt] (unconstrained.south) edge[->, color=blue] (app_curv.north);
\end{tikzpicture}
\caption{We prove three reductions (top row), and use these reductions to obtain specific algorithms and regret bounds (bottom row). Arrows indicate which reductions are used in each algorithm.}
\label{fig:diagram}
\end{figure}

\textbf{Notation.} The dual of a Banach space $B$ over a field $F$, denoted $B^\star$, is the set of all continuous linear maps $B\to F$. We will use the notation $\langle v, w\rangle$ to indicate the application of a dual vector $v\in B^\star$ to a vector $w\in B$. $B^\star$ is also a Banach space with the \emph{dual norm}: $\|v\|_\star=\sup_{w\in B,\ \|v\|=1}\langle w, v\rangle$.
For completeness, in Appendix~\ref{sec:banachspaces} we recall some more background on Banach spaces. 

\section{Online Newton Step to Online Linear Optimization via Betting Algorithms}\label{sec:betfromons}

In this section we show \emph{how to use the Online Newton Step (ONS) algorithm~\citep{hazan2007logarithmic} to construct a 1D parameter-free algorithm}. Our approach relies on the coin-betting abstraction~\citep{orabona2016coin} for the design of parameter-free algorithms. Coin betting strategies record the \emph{wealth} of the algorithm, which is defined by some initial (i.e. user-specified) $\epsilon$ plus the total ``reward'' $\sum_{t=1}^T-g_tw_t$ it has gained:
\begin{equation}
\label{eq:wealth}
\wealth_T=\epsilon-\sum_{t=1}^T g_tw_t~.
\end{equation}
Given this wealth measurement, coin betting algorithms ``bet'' a signed fraction $v_t\in(-1,1)$ of their current wealth on the outcome of the ``coin'' $g_t \in [-1,1]$ by playing $w_{t}=v_t\wealth_{T-1}$, so that $\wealth_T = \wealth_{T-1} - g_t v_t\wealth_{T-1}$. The advantage of betting algorithms lies in the fact that high wealth is equivalent to a low regret~\cite{mcmahan2014unconstrained}, but lower-bounding the wealth of an algorithm is conceptually simpler than upper-bounding its regret because the competitor $\w$ does not appear in \eqref{eq:wealth}. Thus the question is how to pick betting fractions $v_t$ that guarantee high wealth. This is usually accomplished through careful design of bespoke potential functions and meticulous algebraic manipulation, but we take a different and simpler path.

At a high level, our approach is to re-cast the problem of choosing betting fractions $v_t$ as itself an online learning problem. We show that this online learning problem has \emph{exp-concave} losses rather than linear losses. Exp-concave losses are known to be much easier to optimize than linear losses and it is possible to obtain $\ln(T)$ regret rather than the $\sqrt{T}$ limit for linear optimization~\cite{hazan2007logarithmic}. So by using an exp-concave optimization algorithm such as the Online Newton Step (ONS), we find the optimal betting fraction $\v$ very quickly, and obtain high wealth. The pseudocode for the resulting strategy is in Algorithm~\ref{algorithm:ons_1d}.

\begin{algorithm}[h]
\caption{Coin-Betting through ONS}
\label{algorithm:ons_1d}
\begin{algorithmic}[1]
{
    \REQUIRE{Initial wealth $\epsilon>0$}
    \STATE{{\bfseries Initialize: } $\wealth_0=\epsilon$, initial betting fraction $v_1 = 0$}
    \FOR{$t=1$ {\bfseries to} $T$}
    \STATE{Bet $w_t = v_t \, \wealth_{t-1} $, Receive $g_t \in [-1,1]$}
    \STATE{Update $\wealth_{t} = \wealth_{t-1} - g_t w_t $}
    \STATE{//compute new betting fraction $v_{t+1}\in[-1/2,1/2]$ via ONS update on losses $-\ln(1-g_tv)$}
    \STATE{Set $z_t = \tfrac{d}{dv_t}\left(-\ln(1-g_t v_t)\right)=\tfrac{g_t}{1-g_t v_t}$}
    \STATE{Set $A_{t} = 1 + \sum_{i=1}^t z_i^2$}
    \STATE{$v_{t+1} = \max\left(\min\left(v_{t} - \tfrac{2}{2-\ln(3)} \tfrac{z_t}{A_{t}}, 1/2\right),-1/2\right)$}
    \ENDFOR
}
\end{algorithmic}
\end{algorithm}

Later (in Section \ref{sec:onsbanachspace}), we will see that this same 1D argument holds seamlessly in Banach spaces, where now the betting fraction $v_t$ is a vector in the Banach space and the outcome of the coin $g_t$ is a vector in the dual space with norm bounded by 1. We therefore postpone computing exact constants for the Big-O notation in Theorem~\ref{thm:1dregret} to the more general Theorem~\ref{thm:banachregret}.

It is important to note that ONS in 1D is extremely simple to implement. Even the projection onto a bounded set becomes just a truncation between two real numbers, so that Algorithm~\ref{algorithm:ons_1d} can run quickly. We can show the following regret guarantee:
\begin{theorem}
\label{thm:1dregret}
For $|g_t|\le 1$, Algorithm~\ref{algorithm:ons_1d}, guarantees the regret bound:
\begin{align*}
R_T(\w) &=O\left[\epsilon + \max\left(|\w|\ln\left[\frac{|\w|\sum_{t=1}^T g_t^2}{\epsilon}\right]\ ,\   |\w|\sqrt{\sum_{t=1}^T g_t^2 \ln\left[\frac{|\w|^2\sum_{t=1}^T g_t^2}{\epsilon^2}+1\right]}\right)\right]~.
\end{align*}
\end{theorem}

\begin{proof}
Define $\wealth_T(\v)$ to be wealth of the betting algorithm that bets the constant (signed) fraction $\v$ on every round, starting from initial wealth $\epsilon>0$.

We begin with the regret-reward duality that is the start of all coin-betting analyses~\cite{orabona2016coin}. Suppose that we obtain a bound $\wealth_T\ge f_T\left(-\sum_{t=1}^T g_t\right)$ for some $f_T$. Then,
\vspace{-0.1cm}
\[
R_T(\w)-\epsilon
= - \wealth_T-\sum_{t=1}^T g_t\w
\le -\sum_{t=1}^T g_t\w-f_T\left(-\sum_{t=1}^T g_t\right)
\le \sup_{G\in \R} \ G\w - f_T(G)
=f_T^\star(\w),
\vspace{-0.1cm}
\]
where $f_T^\star$ indicates the Fenchel conjugate, defined by $f_T^\star(x) = \sup_{\theta} \ \theta x - f_T(\theta)$.

So, now it suffices to prove a wealth lower bound. First, observing that $\wealth_T= \wealth_{T-1}-\wealth_{T-1} g_t v_t$, we derive a simple expression for $\ln \wealth_T$ by recursion:
\vspace{-0.2cm}
\[
\ln \wealth_T
=\ln \left(\wealth_{T-1}(1-g_t v_t)\right)=\ln(\epsilon)+\sum_{t=1}^T \ln(1-v_tg_t)~.
\vspace{-0.1cm}
\]
Similarly, we have $\ln \wealth_T(\v) = \ln(\epsilon)+\sum_{t=1}^T \ln(1-\v g_t)$. We subtract the identities to obtain
\begin{equation}
\ln \wealth_T(\v) - \ln \wealth_T = \sum_{t=1}^T -\ln(1-v_tg_t)- (-\ln(1-\v g_t))~.\label{eqn:wealthregret}
\end{equation}
Now, the key insight of this analysis: we interpret equation \eqref{eqn:wealthregret} as the regret of an algorithm playing $v_t$ on losses $\ell_t(v) = -\ln(1-vg_t)$, so that we can write
\begin{equation}
\ln \wealth_T = \ln \wealth_T(\v) - R_T^{v}(\v),\label{eqn:logdiff}
\end{equation}
where $R_T^{v}(\v)$ is the regret of our method for choosing $v_t$.

For the next step, observe that $-\ln(1- g_t v)$ is exp-concave (a function $f$ is exp-concave if $\exp(-f)$ is concave), so that choosing $v_t$ is an online exp-concave optimization problem. Prior work on exp-concave optimization allows us to obtain $R^v_T(\v)=O\left(\ln\left(\sum_{t=1}^T g_t^2\right)\right)$ for any $|\v|\le \tfrac{1}{2}$ using the ONS algorithm. Therefore (dropping all constants for simplicity), we use (\ref{eqn:logdiff}) to obtain $\wealth_T \ge \wealth_T(\v) / \sum_{t=1}^T g_t^2$ for all $|\v|\le \tfrac{1}{2}$.

Finally, we need to show that there exists $\v$ such that $\wealth_T(\v)/\sum_{t=1}^T g_t^2$ is high enough to guarantee low regret on our original problem. Consider $\v = \tfrac{-\sum_{t=1}^T g_t}{2\sum_{t=1}^T g_t^2 + 2\left|\sum_{t=1}^T g_t\right|}\in[-1/2,1/2]$. Then, we invoke the tangent bound $\ln(1+x)\ge x-x^2$ for $x\in[-1/2,1/2]$ (e.g. see \citep{cesa2006prediction}) to see:
\begin{align*}
\ln \wealth_T(\v)-\ln(\epsilon)
&=\sum_{t=1}^T \ln(1 - g_t \v )
\ge -\sum_{t=1}^T g_t\v-\sum_{t=1}^T ( g_t \v )^2
\ge\tfrac{\left(\sum_{t=1}^T g_t\right)^2}{4\sum_{t=1}^T g_t^2 + 4\left|\sum_{t=1}^T g_t\right|}~.
\end{align*}
\[
\wealth_T 
\ge \epsilon \exp\left[\tfrac{\left(\sum_{t=1}^T g_t\right)^2}{4\sum_{t=1}^T g_t^2 + 4\left|\sum_{t=1}^T g_t\right|}\right] \bigg/ \sum_{t=1}^T g_t^2
=f_T\left(\sum_{t=1}^T g_t\right),
\]
where $f_T(x) = \epsilon\exp[x^2/(4\sum_{t=1}^T g_t^2 + 4|x|)]/\sum_{t=1}^T g_t^2$. To obtain the desired result, we recall that $\wealth_T\ge f_T\left(\sum_{t=1}^T g_t\right)$ implies $R_T(\w)\le \epsilon + f_T^\star(\w)$, and calculate $f_T^\star$ (see Lemma~\ref{lemma:fenchel_exp2}).

In order to implement the algorithm, observe that our reference betting fraction $\v$ lies in $[-1/2,1/2]$, so we can run ONS restricted to the domain $[-1/2,1/2]$. Exact constants can be computed by substituting the constants coming from the ONS regret guarantee, as we do in Theorem~\ref{thm:banachregret}.
\end{proof}
\vspace{-0.3cm}
\section{From 1D Algorithms to Dimension-Free Algorithms}\label{sec:1dred}

A common strategy for designing parameter-free algorithms is to first create an algorithm for 1D problems (as we did in the previous section), and then invoke some particular algorithm-specific analysis to extend the algorithm to high dimensional spaces~\citep{orabona2016coin, cutkosky2016online, mcmahan2014unconstrained}. This strategy is unappealing for a couple of reasons. First, these arguments are often somewhat tailored to the algorithm at hand, and so a new argument must be made for a new 1D algorithm (indeed, it is not clear that any prior dimensionality extension arguments apply to our Algorithm \ref{algorithm:ons_1d}). Secondly, all such arguments we know of apply only to Hilbert spaces and so do not allow us to design algorithms that consider norms other than the standard Euclidean $2$-norm. In this section we address both concerns by providing a \emph{black-box reduction from optimization in any Banach space to 1D optimization}. In further contrast to previous work, our reduction can be proven in just a few lines.

Our reduction takes two inputs: an algorithm $\onedol$ that operates with domain $\R$ and achieves regret $R^1_T(\w)$ for any $\w\in \R$, and an algorithm $\bol$ that operates with domain equal to the unit ball $S$ in some Banach space $B$, $S=\{x\in B\ :\ \|x\|\le 1\}$ and obtains regret $R^{\bol}_T(\w)$ for any $\w\in S$. In the case when $B$ is $\R^d$ or a Hilbert space, then online gradient descent with adaptive step sizes can obtain $R^{\bol}_T(\w)=\sqrt{2\sum_{t=1}^T \|g_t\|^2_2}$ (which is independent of $\w$)~\citep{hazan2008adaptive}.

Given these inputs, the reduction uses the 1D algorithm $\onedol$ to learn a ``magnitude'' $z$ and the unit-ball algorithm $\bol$ to learn a ``direction'' $y$. This direction and magnitude are multiplied together to form the final output $w=zy$. Given a gradient $g$, the ``magnitude error'' is given by $\langle g, y\rangle$, which is intuitively the component of the gradient parallel to $w$. The ``direction error'' is just $g$. Our reduction is described formally in Algorithm \ref{alg:onedimred}.

\begin{algorithm}[h]
   \caption{One Dimensional Reduction}
   \label{alg:onedimred}
\begin{algorithmic}[1]
   \REQUIRE 1D Online learning algorithm $\onedol$, Banach space $B$ and Online learning algorithm $\bol$ with domain equal to unit ball $S\subset B$
   \FOR{$t=1$ {\bfseries to} $T$}
   \STATE Get point $z_t\in \R$ from $\onedol$
   \STATE Get point $y_t\in S$ from $\bol$
   \STATE Play $w_t = z_ty_t\in B$, receive subgradient $g_t$
   \STATE Set $s_t = \langle g_t, y_t\rangle$
   \STATE Send $s_t$ as the $t$th subgradient to $\onedol$
   \STATE Send $g_t$ as the $t$th subgradient to $\bol$
   \ENDFOR
\end{algorithmic}
\end{algorithm}

\begin{theorem}
Suppose $\bol$ obtains regret $R^{\bol}_T(\w)$ for any competitor $\w$ in the unit ball and $\onedol$ obtains regret $R^1_T(\w)$ for any competitor $\w\in\R$. Then Algorithm \ref{alg:onedimred} guarantees regret:
\[
R_T(\w) \le R^1_T(\|\w\|) + \|\w\|R^{\bol}_T(\w/\|\w\|)~.
\]
Where by slight abuse of notation we set $\w/\|\w\|=0$ when $\w=0$. Further, the subgradients $s_t$ sent to $\onedol$ satisfy $|s_t|\le \|g_t\|_\star$.
\end{theorem}
\begin{proof}
First, observe that $|s_t|\le \|g_t\|_\star\|y_t\|\le \|g_t\|_\star$ since $\|y_t\|\le 1$ for all $t$. Now, compute:
\begin{align*}
R_T(\w)&=\sum_{t=1}^T \langle g_t, w_t-\w\rangle
=\sum_{t=1}^T \langle g_t, z_ty_t\rangle - \langle g_t, \w\rangle\\
&=\sum_{t=1}^T \underbrace{\langle g_t, y_t\rangle z_t - \langle g_t, y_t\rangle \|\w\|}_{\text{regret of }\onedol\text{ at }\|\w\|\in \R} + \langle g_t, y_t\rangle \|\w\|  - \langle g_t, \w\rangle\\
&\le R^1_T(\|\w\|) +\|\w\|\sum_{t=1}^T\underbrace{ \langle g_t, y_t\rangle  -\langle g_t,\w/\|\w\|\rangle}_{\text{regret of }\bol\text{ at }\w/\|w\|\in S}\\
&\le R^1_T(\|\w\|) + \|\w\|R^{\bol}_T(\w/\|\w\|),
\end{align*}
\end{proof}

With this reduction in hand, designing dimension-free and parameter-free algorithms is now exactly as easy as designing 1D algorithms, so long as we have access to a unit-ball algorithm $\bol$. As mentioned, for any Hilbert space we indeed have such an algorithm. In general, algorithms $\bol$ exist for most other Banach spaces of interest \cite{srebro2011universality}, and in particular one can achieve $R^{\bol}_T(\w)\le O\left(\sqrt{\tfrac{1}{\lambda}\sum_{t=1}^T \|g_t\|_\star^2}\right)$ whenever $B$ is $(2,\lambda)$-uniformly convex~\cite{Pinelis15} using the Follow-the-Regularized-Leader algorithm with regularizers scaled by $\tfrac{\sqrt{\lambda}}{\sqrt{\sum_{i=1}^t \|g_i\|_\star^2}}$~\citep{mcmahan2017survey}. 

Applying Algorithm \ref{alg:onedimred} to our 1D Algorithm \ref{algorithm:ons_1d}, for any $(2,\lambda)$-uniformly convex $B$, we obtain:
\begin{align*}
R_T(\w)&=O\left[\|\w\|\max\left(\ln\frac{\|\w\|\sum_{t=1}^T \|g_t\|_\star^2}{\epsilon},\ \sqrt{\sum_{t=1}^T \|g_t\|_\star ^2\ln\left(\frac{\|\w\|^2\sum_{t=1}^T \|g_t\|_\star^2}{\epsilon^2}+1\right)}\right)\right.\\
&\quad\quad\quad\quad\left. + \frac{\|\w\|}{\sqrt{\lambda}}\sqrt{\sum_{t=1}^T \|g_t\|_\star^2} + \epsilon\right]~.
\end{align*}
Spaces that satisfy this property include Hilbert spaces such as $\R^d$ with the $2$-norm (in which case $\lambda=1$), as well the $\R^d$ with the $p$-norm for $p\in (1,2]$ (in which case $\lambda=p-1$). Finally, observe that the runtime of this reduction is equal to the runtime of $\onedol$ plus the runtime of $\bol$, which in many cases (including $\R^d$ with $2$-norm or Hilbert spaces) is the same as online gradient descent.

Not only does this provide the fastest known parameter-free algorithm for an arbitrary norm, it is also the first parameter-free algorithm to obtain a dependence on the gradients of $\|g_t\|_\star^2$ rather than $\|g_t\|_\star$\footnote{Independently, \citep{foster2018online} achieved the same runtime in the supervised prediction setting, but with no adaptivity to $g_t$.}. This improved bound immediately implies much lower regret in easier settings, such as smooth losses with small loss values at $\w$~\citep{srebro2010smoothness}.

\section{Reduction to Constrained Domains}\label{sec:constrained}

The previous algorithms have dealt with optimization over an entire vector space. Although common and important case in practice, sometimes we must perform optimization with constraints, in which each $w_t$ and the comparison point $\w$ must lie in some convex domain $W$ that is not an entire vector space. This constrained problem is often solved with the classical Mirror Descent~\cite{zinkevich2003online} or Follow-the-Regularized-Leader~\cite{shalev2007online} analysis. 
However, these approaches have drawbacks: for unbounded sets, they typically maintain regret bounds that have suboptimal dependence on $\w$, or, for bounded sets, they depend explicitly on the diameter of $W$. We will address these issues with a simple reduction.
Given any convex domain $V\supset W$ and an algorithm $\ol$ that maintains regret $R^{\ol}_T(\w)$ for any $\w\in V$, we obtain an algorithm that maintains $2R^{\ol}_T(\w)$ for any $\w$ in $W$.

Before giving the reduction, we define the distance to a convex set $W$ as $S_W(x) = \inf_{d \in W} \|x-d\|$ as well as the projection to $W$ as $\Pi_W(x)=\{ d \in W: \|d-x\|\leq \|c-x\|, \forall c \in W\}$. Note that if $B$ is reflexive,\footnote{All Hilbert spaces and finite-dimensional Banach spaces are reflexive.} $\Pi_W(x)\neq \emptyset$ and that it is a singleton if $B$ is a Hilbert space~\cite[Exercise 4.1.4]{Lucchetti06}.

The intuition for our reduction is as follows: given a vector $z_t\in V$ from $\ol$, we predict with any $w_t\in \Pi_W(z_t)$. Then give $\ol$ a subgradient at $z_t$ of the surrogate loss function $\langle g_t,\cdot\rangle + \|g_t\|_\star S_W$, which is just the original linearized loss plus a multiple of $S_W$. The additional term $S_W$ serves as a kind of Lipschitz barrier that penalizes $\ol$ for predicting with any $z_t\notin W$. Pseudocode for the reduction is given in Algorithm~\ref{alg:constrained}.

\begin{algorithm}[h]
   \caption{Constraint Set Reduction}
   \label{alg:constrained}
\begin{algorithmic}[1]
   \REQUIRE Reflexive Banach space $B$, Online learning algorithm $\ol$ with domain $V\supset W\subset B$
   \FOR{$t=1$ {\bfseries to} $T$}
   \STATE Get point $z_t\in V$ from $\ol$
   \STATE Play $w_t \in \Pi_W(z_t)$, receive $g_t \in \partial \ell_t(w_t)$
   \STATE Set $\tilde \ell_t(x) = \tfrac{1}{2}\left(\langle g_t, x\rangle + \|g_t\|_\star S_W(x)\right)$
   \STATE Send $\tilde g_t\in \partial \tilde \ell_t(z_t)$ as $t$th subgradient to $\ol$
   \ENDFOR
\end{algorithmic}
\end{algorithm}

\begin{theorem}\label{thm:constrained}
Assume that the algorithm $\ol$ obtains regret $R^{\ol}_T(\w)$ for any $\w \in V$. Then  Algorithm~\ref{alg:constrained} guarantees regret:
\[
R_T(\w)=\sum_{t=1}^T \langle g_t, w_t -  \w\rangle \leq 2R^{\ol}_T(\w), \quad \forall \w \in W~.
\]
Further, the subgradients $\tilde g_t$ sent to $\ol$ satisfy $\|\tilde g_t\|_\star\le \|g_t\|_\star$.
\end{theorem}

Before proving this Theorem, we need a small technical Proposition, proved in Appendix~\ref{sec:sd_convex}.
\begin{restatable}{Proposition}{sdconvex}
\label{prop:sd_convex}
$S_W$ is convex and $1$-Lipschitz for any closed convex set $W$ in a reflexive Banach space $B$.
\end{restatable}
\begin{proof}[of Theorem~\ref{thm:constrained}]
From Proposition~\ref{prop:sd_convex}, we observe that since $S_W$ is convex and $\|g_t\|_\star\ge 0$, $\tilde \ell_t$ is convex for all $t$. Therefore, by $\ol$'s regret guarantee, we have
\[
\sum_{t=1}^T \tilde \ell_t(z_t)-\tilde \ell_t(\w)\le R_T^{\ol}(\w)~.
\]
Next, since $\w\in W$, $\langle g_t,\w\rangle= 2\tilde \ell_t(\w)$ for all $t$. Further, since $w_t\in \Pi_W(z_t)$, we have $\langle g_t, z_t\rangle + \|g_t\|_\star\|w_t-z_t\| = 2\tilde\ell_t(z_t)$. Finally, by the definition of dual norm we have
\[
\langle g_t, w_t-\w\rangle \le \langle g_t, z_t-\w\rangle + \|g_t\|_\star\|w_t-z_t\|=2\tilde \ell_t(z_t)-2\tilde \ell_t(\w)~.
\]
Combining these two lines proves the regret bound of the theorem. The bound on $\|\tilde{g}_t\|_\star$ follows because $S_W$ is 1-Lipschitz, from Proposition~\ref{prop:sd_convex}.
\end{proof}

We conclude this section by observing that in many cases it is very easy to compute an element of $\Pi_W$ and a subgradient of $S_W$. For example, when $W$ is a unit ball, it is easy to see that $\Pi_W(x)=\tfrac{x}{\|x\|}$ and $\partial S_W(x) = \partial \|x\|$ for any $x$ not in the ball. In general, we provide the following result that often simplifies computing the subgradient of $S_W$ (proved in Appendix~\ref{sec:sd_convex}):
\begin{restatable}{theorem}{fixedproj}\label{thm:fixedproj}
Let $B$ be a reflexive Banach space such that for every $0\ne b\in B$, there is a unique dual vector $b^\star$ such that $\|b^\star\|_\star = 1$ and $\langle b^\star ,b\rangle = \|b\|$. Let $W\subset B$ a closed convex set. Given $x\in B$ and $x\notin W$, let $p\in\Pi_W(x)$. Then $\{(x-p)^\star\} = \partial S_W(x)$.
\end{restatable}

\section{Reduction for Multi-Scale Experts}\label{sec:multiscale}

In this section, we apply our reductions to the multi-scale experts problem considered in \citep{foster2017parameter, bubeck2017online}. Our algorithm improves upon both prior algorithms: the approach of \citep{bubeck2017online} has a mildly sub-optimal dependence on the prior distribution, while the approach of \citep{foster2017parameter} takes time $O(T)$ per update, resulting in a quadratic total runtime. Our algorithm matches the regret bound of \citep{foster2017parameter} while running in the same time complexity as online gradient descent.

The multi-scale experts problem is an online linear optimization problem over the probability simplex $\{x\in\R_{\ge 0}^N\ : \sum_{i=1}^N x_i=1\}$ with linear losses $\ell_t(w)=g_t\cdot w$ such that each $g_t=(g_{t,1},\dots,g_{t,N})$ satisfies $|g_{t,i}|\le c_i$ for some known quantities $c_i$. The objective is to guarantee that the regret with respect to the $i$th basis vector $e_i$ (the $i$th ``expert'') scales with $c_i$. Formally, we want $R_T(\w) =O(\sum_{i=1}^N c_i|\w_i|\sqrt{T\log(c_i|\w_i|T/\pi_i)})$, given a prior discrete distribution $(\pi_1,\dots,\pi_N)$. As discussed in depth by \cite{foster2017parameter}, such a guarantee allows us to combine many optimization algorithms into one meta-algorithm that converges at the rate of the best algorithm \emph{in hindsight}.

We accomplish this through two reductions. First, given any distribution $(\pi_1,\dots,\pi_N)$ and any family of 1-dimensional OLO algorithms $\ol(\epsilon)$ that guarantees $R(u)\le O\left(\epsilon+|u|\sqrt{\log(|u|T/\epsilon)T}\right)$ on 1-Lipschitz losses for any given $\epsilon$ (such as our Algorithm \ref{algorithm:ons_1d} or many other parameter-free algorithms), we apply the classic ``coordinate-wise updates'' trick~\cite{streeter2010less} to generate an $N$-dimensional OLO algorithm with regret $R_T(u) = O\left(\epsilon+\sum_{i=1}^N |u_i|\sqrt{\log\left(|u_i|T/(\epsilon\pi_i)\right)T}\right)$ on losses that are $1$-Lipschitz with respect to the $1$-norm.

\begin{algorithm}[t]
\caption{Coordinate-Wise Updates}
\label{alg:coordinate-wise}
\begin{algorithmic}[1]
\REQUIRE parametrized family of 1-D online learning algorithm $\ol(\epsilon)$, prior $\pi$, $\epsilon>0$
\STATE {\bfseries Initialize: } $N$ copies of $\ol$: $\ol_1(\epsilon\pi_1),\dots,\ol_N(\epsilon\pi_N)$
\FOR{$t=1$ {\bfseries to} $T$}
\STATE Get points $z_{t,i}$ from $\ol_i$ for all $i$ to form vector $z_t = (z_{t,1},\dots,z_{t,N})$
\STATE Play $z_t$, get loss $g_t\in \R^N$ with $\|g_t\|_\infty\le 1$
\STATE Send $g_{t,i}$ to $\ol_i$ for all $i$
\ENDFOR
\end{algorithmic}
\end{algorithm}

\begin{theorem}\label{thm:coordinatewise}
Suppose for any $\epsilon>0$, $\ol(\epsilon)$ guarantees regret
\[
R_T(u)\le O\left(\epsilon + |u|\sqrt{\log\left(\tfrac{|u|T}{\epsilon}+1\right)T}\right)
\]
for 1-dimensional losses bounded by $1$. Then Algorithm~\ref{alg:coordinate-wise} guarantees regret
\[
R_T(u) \le O\left(\epsilon + \sum_{i=1}^N |u_i|\sqrt{\log\left(\tfrac{|u_i|T}{\epsilon\pi_i}+1\right)T}\right)~.
\]
\end{theorem}
\begin{proof}
Let $R^i_T(u_i)$ be the regret of the $i$th copy of $\ol$ with respect to $u_i\in \R$. Then
\[
\sum_{t=1}^T \langle g_t, w_t-u\rangle
=\sum_{i=1}^N\sum_{t=1}^T g_{t,i}(w_{t,i}-u_i)
\le \sum_{i=1}^N R^i_T(u_i)\le O\left(\epsilon + \sum_{i=1}^N |u_i|\sqrt{\log\left(\tfrac{|u_i|T}{\epsilon\pi_i}+1\right)T}\right)~.
\]
\end{proof}

\begin{algorithm}[h!]
   \caption{Multi-Scale Experts}
   \label{alg:multiscale}
\begin{algorithmic}[1]
   \REQUIRE parametrized 1-D Online learning algorithm $\ol(\epsilon)$, prior $\pi$, scales $c_1,\dots,c_N$
   \STATE {\bfseries Initialize: } coordinate-wise algorithm $\ol_\pi$ with prior $\pi$ using $\ol(\epsilon)$
   \STATE Define $W=\{x:x_i\ge 0\text{ for all }i\text{ and }\sum_{i=1}^N x_i/c_i=1\}$
   \STATE Let $\ol^W_\pi$ be the result of applying the unconstrained-to-constrained reduction to $\ol_\pi$ with constraint set $W$ using $\|\cdot\|_1$
   \FOR{$t=1$ {\bfseries to} $T$}
   \STATE Get point $z_t\in W$ from $\ol^W_\pi$
   \STATE Set $x_t\in \R^N$ by $x_{t,i}=z_{t,i}/c_i$. Observe that $x_t$ is in the probability simplex
   \STATE Play $x_t$, get loss vector $g_t$
   \STATE Set $\tilde g_t\in \R^N$ by $\tilde g_{t,i} = \tfrac{g_{t,i}}{c_i}$
   \STATE Send $\tilde g_t$ to $\ol^W_\pi$
   \ENDFOR
\end{algorithmic}
\end{algorithm}

With this in hand, notice that applying our reduction Algorithm~\ref{alg:constrained} with the $1$-norm easily yields an algorithm over the probability simplex $W$ with the same regret (up to a factor of 2), as long as $\|g_t\|_\infty\le 1$. Then, we apply an affine change of coordinates to make our multi-scale experts losses have $\|g_t\|_\infty\le 1$, so that applying this algorithm yields the desired result (see Algorithm \ref{alg:multiscale}).

\begin{theorem}\label{thm:multiscale}
If $g_t$ satisfies $|g_{t,i}|\le c_i$ for all $t$ and $i$ and $\ol(\epsilon)$ satisfies the conditions of Theorem~\ref{thm:coordinatewise},
then, for any $\w$ in the probability simplex, Algorithm~\ref{alg:multiscale} satisfies the regret bound
\[
R_T(\w) \le O\left(\epsilon + \sum_{i=1}^N c_i|\w_i|\sqrt{\log\left(\tfrac{c_i|\w_i|T}{\epsilon\pi_i}+1\right)T}\right)~.
\]
\end{theorem}
\begin{proof}
Given any $\w$ in the probability simplex, define $\tilde w\in \R^N$ by $\tilde w_i = c_i\w_i$. Observe that $\tilde w\in W$. Further, observe that since $|g_{t,i}|\le c_i$, $\|\tilde g_t\|_\infty \le 1$. Finally, observe that $\tilde g_t\cdot z_t =\sum_{i=1}^N \tilde g_{t,i} z_{t,i}=\sum_{i=1}^N \tfrac{g_{t,i}}{c_i} c_ix_{t,i} = g_t\cdot x_t$ and similarly $\tilde g_t\cdot \tilde w = g_t\cdot \w$. Thus $\sum_{t=1}^T \tilde g_t\cdot z_t-\tilde g_t\cdot \tilde w=\sum_{t=1}^T g_t\cdot(x_t-\w)$. Now, by Theorem \ref{thm:coordinatewise} and Theorem \ref{thm:constrained} we have
\[
\sum_{t=1}^T g_t\cdot(x_t-\w)=\sum_{t=1}^T \tilde g_t\cdot (z_t- \tilde w)
\le O\left(\epsilon + \sum_{i=1}^N |\tilde w_i|\sqrt{\log\left(\tfrac{|\tilde w_i|T}{\epsilon\pi_i}+1\right)T}\right)
\]

Now simply substitute the definition $\tilde w_i = c_i\w_i$ to complete the proof.
\end{proof}

In Appendix~\ref{sec:simplexprojection} we show how to compute the projection $\Pi_S$ and a subgradient of $S_W$ in $O(N)$ time via a simple greedy algorithm. As a result, our entire reduction runs in $O(N)$ time per update.

\section{Reduction to Adapt to Curvature}\label{sec:metagrad}

In this section, we present a black-box reduction to make a generic online learning algorithm over a Banach space adaptive to the curvature of the losses. Given a set $W$ of diameter $D=\sup_{x,y\in W}\|x-y\|$, our reduction obtains $O(\log(TD)^2/\mu)$ regret on online $\mu$-strongly convex optimization problems, but still guarantees $O(\log(TD)^2D\sqrt{T})$ regret for online linear optimization problems, both of which are only log factors away from the optimal guarantees. We follow the intuition of \citep{cutkosky2017stochastic}, who suggest adding a weighted average of previous $w_t$s to the outputs of a base algorithm as a kind of ``momentum'' term. We improve upon their regret guarantee by a log factor and by the $\|g_t\|_\star^2$ terms instead of $\|g_t\|_\star$. More importantly, their algorithm involves an optimization step which may be very slow for most domains (e.g. the unit ball). In contrast, thanks to our fast reduction in Section \ref{sec:constrained}, we keep the same running time as the base algorithm.
Finally, previous results for algorithms with similar regret (e.g. \citep{cutkosky2017stochastic, vanervan2016metagrad}) show logarithmic regret only for \emph{stochastic} strongly convex problems. We give a two-line argument extending this to the adversarial case as well.

\begin{algorithm}[ht!]
\caption{Adapting to Curvature}
\label{alg:metagradreduction}
\begin{algorithmic}[1]
   \REQUIRE{Online learning algorithm $\ol$}
   \STATE {\bfseries Initialize: } $W$, a convex closed set in a reflexive Banach space, $\overline{x}_0$ an arbitrary point in $W$
   \FOR{$t=1$ {\bfseries to} $T$}
   \STATE Get point $w_t$ from $\ol$
   \STATE Set $z_t = w_t+\overline{x}_{t-1}$
   \STATE Play $x_t \in \Pi_W(z_t)$, receive subgradient $g_t\in\partial \ell_t(x_t)$
   \STATE Set $\tilde g_t \in g_t+ \|g_t\|_\star  \partial S_W(z_t)$
   \STATE Set $\overline{x}_t = \tfrac{\overline{x}_0+\sum_{i=1}^t \|\tilde {g}_{i}\|_\star ^2 x_{i}}{1+\sum_{i=1}^t \|\tilde{g}_i\|_\star ^2}$
   \STATE Send $\tilde{g}_t$ so $\ol$ as the $t$th subgradient
   \ENDFOR
\end{algorithmic}
\end{algorithm}

\begin{restatable}{theorem}{metagrad}\label{thm:metagrad}
Let $\ol$ be an online linear optimization algorithm that outputs $w_t$ in response to $g_t$.
Suppose $W$ is a convex closed set of diameter $D$. Suppose $\ol$ guarantees for all $t$ and $\v$:
\begin{align*}
\sum_{i=1}^t \langle \tilde{g}_i,w_i-\v\rangle &\le \epsilon + \|\v\|A\sqrt{\sum_{i=1}^t \|\tilde{g}_i\|_\star^2\left(1+\ln\left(\tfrac{\|\v\|^2t^C}{\epsilon^2}+1\right)\right)} + B\|\v\|\ln\left(\tfrac{\|\v\|t^C}{\epsilon}+1\right),
\end{align*}
for constants $A$, $B$ and $C$ and $\epsilon$ independent of $t$.
Then for all $\w \in W$, Algorithm~\ref{alg:metagradreduction} guarantees
\[
R_T(\w) \le \sum_{t=1}^T\langle g_t,x_t-\w\rangle
\le O\left(\sqrt{V_T(\w)\ln\tfrac{T D}{\epsilon}\ln(T)} + \ln\tfrac{DT}{\epsilon}\ln(T)+\epsilon\right),
\]
where $V_T(\w) := \|\overline{x}_0-\w\|^2+\sum_{t=1}^T\|\tilde{g}_t\|_\star ^2\|x_t-\w\|^2\le D^2+\sum_{t=1}^T\|g_t\|_\star ^2\|x_t-\w\|^2$.
\end{restatable}
To see that Theorem \ref{thm:metagrad} implies logarithmic regret on online strongly-convex problems, suppose that each $\ell_t$ is $\mu$-strongly convex, so that $\ell_t(w_t)-\ell(\w)\le \langle g_t, w_t-\w\rangle - \tfrac{\mu}{2}\|w_t-\w\|^2$. Then:
\begin{align*}
\sum_{t=1}^T \ell(x_t)-\ell(\w)&\le O\left(\sqrt{\log^2(DT)\sum_{t=1}^T \|x_t-\w\|^2} -\frac{\mu}{2}\sum_{t=1}^T\|x_t-\w\|^2+ \log^2(TD)\right)\\
&\le O\left(\sup_{X}\sqrt{\log^2(DT)X} -\frac{\mu}{2}X+ \log^2(TD)\right)=O\left(\log^2(DT)\left(1+\frac{1}{\mu}\right)\right)~.
\end{align*}
Where we have used $\|g_t\|_\star\le 1$.

\section{Banach-space betting through ONS}\label{sec:onsbanachspace}

In this section, we present the Banach space version of the one-dimensional Algorithm~\ref{algorithm:ons_1d}. The pseudocode is in Algorithm~\ref{algorithm:ons}. We state the algorithm in its most general Banach space formulation, which obscures some of its simplicity in more common scenarios. For example, when $B$ is $\R^d$ equipped with the $p$-norm, then the linear operator $L$ can be taken to be simply the identity map $I:\R^d\to\R^d\cong (\R^d)^\star$, and the ONS portion of the algorithm is the standard $d$-dimensional ONS algorithm. We give the regret guarantee of Algorithm \ref{algorithm:ons} in Theorem \ref{thm:banachregret}. The proof, modulo technical details of ONS in Banach spaces, is identical to Theorem \ref{thm:1dregret}, and can be found in Appendix~\ref{sec:proof_regret_ons}. 

\begin{algorithm}[ht!]
\caption{Banach-space betting through ONS}
\label{algorithm:ons}
\begin{algorithmic}[1]
{
    \REQUIRE{Real Banach space $B$, initial linear operator $L:B\to B^\star$, initial wealth $\epsilon>0$}
    \STATE{{\bfseries Initialize: } $\wealth_0=\epsilon$, initial betting fraction $v_1 = 0 \in S=\{x\in B\ : \|x\|\le \tfrac{1}{2}\}$}
    \FOR{$t=1$ {\bfseries to} $T$}
    \STATE{Bet $w_t = v_t \, \wealth_{t-1} $, receive $g_t$, with $\|g_t\|_\star  \leq 1$}
    \STATE{Update $\wealth_{t} = \wealth_{t-1} - \langle g_t, w_t\rangle $}
    \STATE{//compute new betting fraction $v_{t+1}\in S$ via ONS update on losses $-\ln(1-\langle g_t, v\rangle)$:}
    \STATE{Set $z_t = \tfrac{d}{dv_t}\left(-\ln(1-\langle g_t, v_t\rangle)\right)=\tfrac{g_t}{1-\langle g_t, v_t\rangle}$}
    \STATE{Set $A_{t}(x) = L(x) + \sum_{i=1}^t z_i \langle z_i, x\rangle$}
    \STATE{$v_{t+1} = \Pi^{A_{t}}_S (v_{t} - \tfrac{2}{2-\ln(3)} A^{-1}_{t}(z_t))$, where $\Pi^{A_t}_S(x)=\argmin_{y \in S} \ \langle A_t(y-x), y-x\rangle$}
    \ENDFOR
}
\end{algorithmic}
\end{algorithm}

\begin{theorem}
\label{thm:banachregret}
Let $B$ be a $d$-dimensional real Banach space and $u\in B$ be an arbitrary unit vector. Then, there exists a linear operator $L$ such that using the Algorithm~\ref{algorithm:ons}, we have for any $\w\in B$,
\begin{align*}
R_T(\w) 
&\leq \epsilon+ 
\max\left\{\frac{d\|\w\|}{2}-8\|\w\|+8\|\w\|\ln\left[\frac{8\|\w\|\left(1+4\sum_{t=1}^T \|g_t\|^2_\star\right)^{4.5d}}{\epsilon}\right],\right.\\
&\quad\quad\quad\quad\left.2\sqrt{\sum_{t=1}^T \langle g_t,\w\rangle^2 \ln\left(\frac{5\|\w\|^2}{\epsilon^2}\left(8\sum_{t=1}^T \|g_t\|^2+2\right)^{9d+1}+1\right)}\right\}~.
\end{align*}
\end{theorem}

The main particularity of this bound is the presence of the terms $\sqrt{d \sum_{t=1}^T \langle g_t,\w\rangle^2}$ rather than the usual $\|\w\|\sqrt{\sum_{t=1}^T \|g_t\|_\star ^2}$. We can interpret this bound as being adaptive to any \emph{sequence} of norms $\|\cdot\|_1,\dots,\|\cdot\|_t$ because $\sqrt{d \sum_{t=1}^T \langle g_t,\w\rangle^2}\le \sqrt{d\sum_{t=1}^T \|\w\|_t^2(\|g_t\|_t)_\star^2}$. A similar kind of ``many norm adaptivity'' was recently achieved in \cite{foster2017parameter}, which competes with the best \emph{fixed} $L_p$ norm (or the best fixed norm in any finite set). Our bound in Theorem~\ref{thm:banachregret} is a factor of $\sqrt{d}$ worse,\footnote{The dependence on $d$ is unfortunately unimprovable, as shown by \citep{luo2016efficient}.} but we can compete with any possible sequence of norms rather than with any fixed one.

Similar regret bounds to our Theorem~\ref{thm:banachregret} have already appeared in the literature. The first one we are aware of is the Second Order Perceptron~\cite{cesa2005second} whose \emph{mistake bound} is exactly of the same form. Recently, a similar bound was also proven in \cite{NIPS2017_7060}, under the assumption that $W$ is of the form $W=\{\v: \langle g_t, \v\rangle \leq C\}$, for a known $C$. Also, \citet{kotlowski2017scale} proved the same bound when the losses are of the form $\ell_t(w_t)=\ell(y_t, w_t \cdot x_t)$ and the algorithm receives $x_t$ before its prediction. In contrast, we can deal with unbounded $W$ and arbitrary convex losses through the use of subgradients. Interestingly, all these algorithms (including ours) have a $O(d^2)$ complexity per update.

\section{Conclusions}

We have introduced a sequence of three reductions showing that parameter-free online learning algorithms can be obtained from online exp-concave optimization algorithms, that optimization in a vector space with any norm can be obtained from 1D optimization, and that online optimization with constraints is no harder than optimization without constraints. Our reductions result in simpler arguments in many cases, and also often provide better algorithms in terms of regret bounds or runtime. We therefore hope that these tools will be useful for designing new online learning algorithms.

\section*{Acknowledgments}
This material is based upon work partly supported by the National Science Foundation under grant no. 1740762 ``Collaborative Research: TRIPODS Institute for Optimization and Learning'' and by a Google Research Award for FO.

{
\small
\bibliographystyle{plainnat}
\bibliography{all}
}

\newpage

\appendix

\section*{Appendix}
This appendix is organized as follows:
\begin{enumerate}
\item In Section~\ref{sec:banachspaces} we collect some background information about Banach spaces, their duals, and other properties.
\item In Section~\ref{sec:onsbanachproof} we provide an analysis of the ONS algorithm in Banach spaces that is useful for proving Theorem~\ref{thm:banachregret}.
\item In Section~\ref{sec:proof_regret_ons} we apply this analysis of ONS in Banach spaces to prove Theorem \ref{thm:banachregret}, and provide the missing Fenchel conjugate calculation required to prove Theorem~\ref{thm:1dregret}, which are our reductions from parameter-free online learning to Exp-concave optimization.
\item In Section~\ref{sec:sd_convex} we prove Proposition~\ref{prop:sd_convex}, used in our reduction from constrained optimization to unconstrained optimization in Section~\ref{sec:constrained}. In this section we also prove Theorem~\ref{thm:fixedproj}, which simplifies computing subgradients of $S_W$ in many cases.
\item In Section~\ref{sec:simplexprojection} we show how to compute $\Pi_W$ and a subgradient of $S_W$ on $O(N)$ time for use in our multi-scale experts algorithm. 
\item Finally, in Section~\ref{sec:metagradproof} we prove Theorem~\ref{thm:metagrad}, our regret bound for an algorithm that adapts to stochastic curvature.
\end{enumerate}

\section{Banach Spaces}
\label{sec:banachspaces}

\begin{Definition}\label{def:banach}
A \emph{Banach space} is a vector space $B$ over $\R$ or $\C$ equipped with a norm $\|\cdot\|:B\to \R$ such that $B$ is complete with respect to the metric $d(x,y) = \|x-y\|$ induced by the norm.
\end{Definition}

Banach spaces include the familiar vector spaces $\R^d$ equipped with the Euclidean $2$-norm, as well as the the same vector spaces equipped with the $p$-norm instead.

An important special case of Banach spaces are the Hilbert spaces, which are Banach spaces that are also equipped with an inner-product $\langle,\rangle:B\times B\to\R$ (a symmetric, positive definite, non-degenerate bilinear form) such that $\langle b, b\rangle = \|b\|^2$ for all $b\in B$. In the complex case, the inner-product is $\C$ valued and the symmetric part of the definition is replaced with the condition $\langle v, w\rangle = \overline{\langle w, v\rangle}$ where $\overline{x}$ indicates complex conjugation. Hilbert spaces include the typical examples of $\R^d$ with the usual dot product, as well as reproducing kernel Hilbert spaces.

The dual of a Banach space $B$ over a field $F$, denoted $B^\star$, is the set of all continuous linear functions $B\to F$. For Hilbert spaces, there is a natural isomorphism $B\cong B^\star$ given by $b\mapsto \langle b,\cdot\rangle$. Inspired by this isomorphism, in general we will use the notation $\langle v, w\rangle$ to indicate application of a dual vector $v\in B^\star$ to a vector $w\in B$. It is important to note that our use of this notation in no way implies the existence of an inner-product on $B$. When $B$ is a Banach space, $B^\star$ is also a Banach space with the \emph{dual norm}: $\|w\|_\star=\sup_{v\in B,\ \|v\|=1}\langle w, v\rangle$. A subgradient of a convex function $\ell:B\to \R$ is naturally an element of the dual $B^\star$. Therefore, the reduction to linear losses by $\ell_t(w_t)-\ell_t(\w)\le \langle g_t,w_t-\w\rangle$ for $g_t\in \partial \ell_t(w_t)$ generalizes perfectly to the case where $W$ is a convex subset of a Banach space.

Given any vector space $V$, there is a natural injection $V\to V^{\star\star}$ given by $x\mapsto \langle\cdot,x\rangle$. When this injection is an isomorphism of Banach spaces, then the space $V$ is called \emph{reflexive}. All finite-dimensional Banach spaces are reflexive.

Given any linear map of Banach spaces $T:X\to Y$, we define the \emph{adjoint} map $T^\star:Y^\star\to X^\star$ by $T^\star(y^\star)(x) = \langle y^\star, T(x)\rangle$. $T^\star$ has the property (by definition) that $\langle y^\star, T(x)\rangle = \langle T^\star(y^\star), x\rangle$. As a special case, if $B$ is a reflexive Banach space and $T:B\to B^\star$, then we can use the natural identification between $B^{\star\star}$ and $B$ to view $T^\star$ as $T^\star:B\to B^\star$. Thus, in this case it is possible to have $T=T^\star$, in which case we call $T$ self-adjoint.

\begin{Definition}\label{thm:uniformconvex}
We define a Banach space $B$ as $(p,D)$ uniformly convex if~\citep{Pinelis15}:
\begin{equation}
\|x+y\|^p + \|x-y\|^p \geq 2\|x\|^p+2D\|y\|^p, \quad \forall x,y\in B~.
\end{equation}
\end{Definition}

From this definition, we can see that if $B$ is $(2,D)$ uniformly convex, then $\|\cdot\|^2$ is a $D$-strongly convex function with respect to $\|\cdot\|$:
\begin{lemma}
\label{lemma:center}
Let $f(x)$ a convex function that satisfies
\[
f\left(\frac{x+y}{2}\right) \leq \frac{1}{2} f(x) + \frac{1}{2}f(y) - \frac{D}{2p} \|x-y\|^p~.
\]
Then, $f$ satisfies $f(x+\delta)\ge f(x) + g(\delta) +D\frac{\|\delta\|^p}{p}$ for any subgradient $g\in \partial f(x)$. In particular for $p=2$, $f$ is $D$ strongly convex with respect to $\|\cdot\|$.
\end{lemma}
\begin{proof}
Set $y = x+2\delta$ for some arbitrary $\delta$. Let $g\in \fX^\star $ be an arbitrary subgradient of $f$ at $x$.
Let $R_x(\tau) = f(x+\tau) - (f(x) + g(\tau))$. Then
\begin{align*}
f(x) + g(\delta) &\le f\left(\frac{x+y}{2}\right)
\le \frac{f(x)+f(x+2\delta)}{2}-\frac{D\|2\delta\|^p}{2p}
=f(x) + g(\delta) + \frac{R_x(2\delta)}{2} - \frac{D\|2\delta\|^p}{2p},
\end{align*}
that implies $\frac{D}{p}\|2\delta\|^p\le R_x(2\delta)$. So that $f(x+\tau)= f(x) + g(\tau) + R_x(\tau) \ge f(x) + g(\tau) + \frac{D}{p}\|\tau\|^p$ as desired.
\end{proof}

\begin{lemma}
Let $B$ be a $(2,D)$ uniformly convex Banach space, then $f(x)=\frac{1}{2}\|x\|^2$ is $D$-strongly convex.
\end{lemma}
\begin{proof}
Let $x=u+v$ and $y=u-v$. Then, from the definition of $(2,D)$ uniformly convex Banach space, we have
\[
2\|u+v\|^2 +2 D \|u-v\|^2 \leq 4 \|u\|^2 + 4 \|v\|^2,
\]
that is
\[
\frac{1}{2}\left\|\frac{u+v}{2}\right\|^2 \leq  \frac{1}{2}\|u\|^2 + \frac{1}{2}\|v\|^2 -\frac{D}{4} \|u-v\|^2~.
\]
Using Lemma~\ref{lemma:center}, we have the stated bound.
\end{proof}

Any Hilbert space is $(2,1)$-strongly convex. As a slightly more exotic example, $\R^d$ equipped with the $p$-norm is $(2,p-1)$ strongly-convex for $p\in(1,2]$.

\section{Proof of the regret bound of ONS in Banach spaces}\label{sec:onsbanachproof}

First, we need some additional facts about self-adjoint operators. These are straight-forward properties in Hilbert spaces, but may be less familiar in Banach spaces so we present them below for completeness.
\begin{Proposition}\label{thm:adjointinverse}
Suppose $X$ and $Y$ are Banach spaces and $T:X\to Y$ is invertible. Then, $T^\star$ is invertible and $(T^{-1})^\star = (T^\star)^{-1}$.
\end{Proposition}
\begin{proof}
Let $y^\star \in Y^\star$. Let $x\in X$. Recall that by definition $\langle T^\star(y^\star), x\rangle = \langle y^\star, T(x)\rangle$. Then we have
\[
\langle (T^{-1})^\star(T^\star (y^\star)), x\rangle
= \langle T^\star(y^\star), T^{-1}(x)\rangle
= \langle y^\star,x\rangle,
\]
where we used the definition of adjoint twice.
Therefore, $(T^{-1})^\star(T^\star (y^\star))=y^\star$ and so $(T^{-1})^\star=(T^\star)^{-1}$.
\end{proof}

\begin{Proposition}\label{thm:outerproduct}
Suppose $B$ is a reflexive Banach space and $T:B\to B^\star$ is such that
\[
T(x)=\sum_{i=1}^N \langle b^i, x\rangle b^i
\]
for some vectors $b^i\in B^\star$. Then $T^\star=T$.
\end{Proposition}
\begin{proof}
Let $g,f\in B$. Since $B$ is reflexive, $g$ corresponds to the function $\langle \cdot, g\rangle\in B^{\star\star}$. Now, we compute:
\[
T^\star(g)(f)
=\langle T(f), g\rangle
=\sum_{i=1}^N \langle b^i,f\rangle \langle b^i,g\rangle
= \langle T(g), f\rangle
= T(g)(f)~.
\]
\end{proof}

\begin{Proposition}\label{thm:normsum}
Suppose $\tau>0$, $B$ is a $d$-dimensional real Banach space, $b^1,\dots,b^d$ are a basis for $B^\star$ and $g_1,\dots,g_T$ are elements of $B^\star$. Then, $A:B\to B^\star$ defined by $A(x) = \tau \sum_{i=1}^d\langle b^i,x\rangle b^i + \sum_{t=1}^T\langle g_t,x\rangle g_t$ is invertible and self-adjoint, and $\langle Ax,x\rangle> 0$ for all $x\ne 0$.
\end{Proposition}
\begin{proof}
First, $A$ is self-adjoint by Proposition \ref{thm:outerproduct}.

Next, we show $A$ is invertible. Suppose otherwise. Then, since $B$ and $B^\star$ are both $d$-dimensional, $A$ must have a non-trivial kernel element $x$. Therefore,
\begin{align}
0&=\langle Ax,x\rangle = \tau \sum_{i=1}^d\langle b^i,x\rangle^2+\sum_{t=1}^T \langle g_t,x\rangle^2,\label{eqn:psd}
\end{align}
so that $\langle b^i,x\rangle =0$ for all $i$. Since the $b^i$ form a basis for $B^\star$, this implies $\langle y,x\rangle=0$ for all $y\in B^\star$, which implies $x=0$. Therefore, $A$ has no kernel and so must be invertible.

Finally, observe that since \eqref{eqn:psd} holds for any $x$, we must have $\langle Ax, x\rangle >0$ if $x\ne 0$. 
\end{proof}

Now we state the ONS algorithm in Banach spaces and prove its regret guarantee:

\begin{algorithm}[t]
\caption{ONS in Banach Spaces}
\label{algorithm:onsinbanachspaces}
\begin{algorithmic}[1]
{
    \REQUIRE{Real Banach space $B$, convex subset $S\subset B$, initial linear operator $L:B\to B^\star$, $\tau,\beta>0$}
    \STATE{{\bfseries Initialize: } $v_1 = 0 \in S$}
    \FOR{$t=1$ {\bfseries to} $T$}
    \STATE{Play $v_t$}
    \STATE{Receive $z_t\in B^\star$}
    \STATE{Set $A_{t}(x) = \tau L(x) + \sum_{i=1}^t z_i \langle z_i, x\rangle$}
    \STATE{$v_{t+1} = \Pi^{A_{t}}_S (v_{t} - \frac{1}{\beta} A^{-1}_{t}(z_t))$, where $\Pi^{A_t}_S(x)=\argmin_{y \in S} \ \langle A_t(y-x), y-x\rangle$}
    \ENDFOR
}
\end{algorithmic}
\end{algorithm}

\begin{theorem}\label{thm:onsfirstpass}
Using the notation of Algorithm~\ref{algorithm:onsinbanachspaces},
suppose $L(x)=\sum_{i=1}^d\langle b^i,x\rangle$ for some basis $b^i\in B^\star$ and that $B$ is $d$-dimensional. Then for any $\v\in S$,
\[
\sum_{t=1}^T \left(\langle z_t, v_t - \v \rangle -\frac{\beta}{2} \langle z_t, v_t -\v\rangle^2\right)
\leq \frac{\beta\tau}{2} \langle L(\v), \v\rangle + \frac{2}{\beta} \sum_{t=1}^T \langle z_t,  A^{-1}_{t}(z_t)\rangle~.
\]
\end{theorem}
\begin{proof}
First, observe by Proposition~\ref{thm:normsum} that $A_t$ is invertible and self-adjoint for all $t$.

Now, define $x_{t+1}=v_t - \frac{1}{\beta} A^{-1}_{t}(z_t)$ so that $v_{t+1}=\Pi^{A_t}_S(x_{t+1})$.
Then, we have
\begin{align*}
x_{t+1} - \v = v_t - \v - \frac{1}{\beta} A^{-1}_{t}(z_t),
\end{align*}
that implies
\begin{align*}
A_{t}(x_{t+1} - \v) = A_{t}(v_t - \v - \frac{1}{\beta} A^{-1}_{t}(z_t)) = A_{t}(v_t - \v) - \frac{1}{\beta} z_t,
\end{align*}
and
\begin{align*}
\langle A_{t}&(x_{t+1} - \v),x_{t+1} - \v\rangle \\
&= \langle A_{t}(v_t - \v) - \frac{1}{\beta} z_t, x_{t+1} - \v\rangle \\
&= \langle A_{t}(v_t - \v), x_{t+1} - \v\rangle - \frac{1}{\beta} \langle z_t, x_{t+1} - \v\rangle \\
&= \langle A_{t}(v_t - \v), x_{t+1} - \v\rangle - \frac{1}{\beta} \langle z_t, v_t - \v - \frac{1}{\beta} A^{-1}_{t}(z_t)\rangle \\
&= \langle A_{t}(v_t - \v), x_{t+1} - \v\rangle - \frac{1}{\beta} \langle z_t, v_t - \v\rangle  + \frac{1}{\beta^2} \langle z_t,  A^{-1}_{t}(z_t)\rangle \\
&= \langle A_{t}(v_t - \v), v_t - \v - \frac{1}{\beta} A^{-1}_{t}(z_t)\rangle - \frac{1}{\beta} \langle z_t, v_t - \v\rangle  + \frac{1}{\beta^2} \langle z_t,  A^{-1}_{t}(z_t)\rangle \\
&= \langle A_{t}(v_t - \v), v_t - \v\rangle - \frac{1}{\beta} \langle A_{t}(v_t - \v), A^{-1}_{t}(z_t)\rangle - \frac{1}{\beta} \langle z_t, v_t - \v\rangle  + \frac{1}{\beta^2} \langle z_t,  A^{-1}_{t}(z_t)\rangle \\
&= \langle A_{t}(v_t - \v), v_t - \v\rangle - \frac{2}{\beta} \langle z_t, v_t - \v\rangle  + \frac{1}{\beta^2} \langle z_t,  A^{-1}_{t}(z_t)\rangle,
\end{align*}
where in the last line we used $\langle A_{t}(v_t - \v), A^{-1}_{t}(z_t)\rangle=\langle (v_t - \v), A_t^\star A^{-1}_{t}(z_t)\rangle$ and $A_t^\star = A_t$.
We now use the Lemma~8 from \cite{hazan2007logarithmic}, extended to Banach spaces thanks to the last statement of Proposition~\ref{thm:normsum}, to have
\[
\langle A_{t}(x_{t+1} - \v),x_{t+1} - \v\rangle \geq \langle A_{t}(v_{t+1} - \v),v_{t+1} - \v\rangle
\]
to have
\begin{align*}
\langle z_t, v_t - \v\rangle 
&\leq \frac{\beta}{2} \langle A_{t}(v_t - \v), v_t - \v\rangle - \frac{\beta}{2} \langle A_{t}(v_{t+1} - \v),v_{t+1} - \v\rangle   + \frac{2}{\beta} \langle z_t,  A^{-1}_{t}(z_t)\rangle~.
\end{align*}
Summing over $t=1,\cdots,T$, we have
\begin{align*}
\sum_{t=1}^T \langle z_t, v_t - \v\rangle 
&\leq \frac{\beta}{2} \langle A_{1}(v_1 - \v), v_1 - \v\rangle + \frac{\beta}{2}\sum_{t=2}^T \langle A_{t}(v_t - \v) -A_{t-1}(v_t - \v), v_t - \v\rangle \\
&\quad - \frac{\beta}{2} \langle A_{T}(v_{T+1} - \v),v_{T+1} - \v\rangle   + \frac{2}{\beta} \sum_{t=1}^T \langle z_t,  A^{-1}_{t}(z_t)\rangle \\
&\leq \frac{\beta}{2} \langle A_{1}(v_1 - \v), v_1 - \v\rangle +\frac{\beta}{2}\sum_{t=2}^T \langle z_t \langle z_t, v_t - \v\rangle , v_t - \v\rangle + \frac{2}{\beta} \sum_{t=1}^T \langle z_t,  A^{-1}_{t}(z_t)\rangle\\
&= \frac{\beta}{2} \langle \tau L(\v), \v\rangle +\frac{\beta}{2}\sum_{t=1}^T \langle z_t , v_t - \v\rangle^2 + \frac{2}{\beta} \sum_{t=1}^T \langle z_t,  A^{-1}_{t}(z_t)\rangle~.
\end{align*}
\end{proof}

It remains to choose $L$ properly and analyze the sum $\sum_{t=1}^T \langle z_t,  A^{-1}_{t}(z_t)\rangle$ In order to do this, we introduce the concept of an Auerbach basis (e.g. see \citep{hajek2007biorthogonal} Theorem 1.16):
\begin{theorem}\label{thm:auerbachexists}
Let $B$ be a $d$-dimensional Banach space. Then there exists a basis of $b_1,\dots,b_d$ of $B$ and a basis $b^1,\dots,b^d$ of $B^\star$ such that $\|b_i\|=\|b^i\|_\star=1$ for all $i$ and $\langle b_i, b^j\rangle = \delta_{ij}$. Any bases $(b_i)$ and $(b^i)$ satisfying these conditions is called an \emph{Auerbach basis}.
\end{theorem}

We will use an Auerbach basis to define $L$, and also to provide a coordinate system that makes it easier to analyze the sum $\sum_{t=1}^T \langle z_t,  A^{-1}_{t}(z_t)\rangle$.
\begin{theorem}\label{thm:useauerbach}
Suppose $B$ is $d$-dimensional. Let $(b_i)$ and $(b^i)$ be an Auerbach basis for $B$. Set $L(x)=\sum_{i=1}^d \langle b^i, x\rangle b^i$. Define $A_t$ as in Algorithm \ref{algorithm:ons}. Then, for any $\v \in S$, the following holds
\begin{align*}
\frac{\beta\tau}{2} \langle L(\v), \v\rangle + \frac{2}{\beta} \sum_{t=1}^T \langle z_t,  A^{-1}_{t}(z_t)\rangle
\le \frac{\beta\tau}{2}d\|\v\|^2 +\frac{2}{\beta}d\ln\left(\frac{\sum_{t=1}^T \|z_t\|_\star ^2}{\tau} + 1\right)~.
\end{align*}
\end{theorem}
\begin{proof}
First, we show that $\frac{\beta}{2}\langle L(\v),\v\rangle \le \frac{\beta d}{2} \|\v\|^2$. To see this, observe that for any $x\in B$,
\begin{align*}
\langle L(x), x\rangle=\sum_{i=1}^d\langle b^i,x\rangle^2\le \sum_{i=1}^d\|b^i\|_\star^2\|x\|^2\le d\|x\|^2~.
\end{align*}

Now, we characterize the sum part of the bound. The basic idea is to use the Auerbach basis to identify $B$ with $\R^d$ (equivalently, we view $\langle L(x),x\rangle$ as an inner product on $B$). We use this identification to translate all quantities in $B$ and $B^\star$ to vectors in $\R^d$, and observe that the 2-norm of any $g_t$ in $\R^d$ is at most $d$. Then we use analysis of the same sum terms in the classical analysis of ONS in $\R^d$ \citep{hazan2007logarithmic} to prove the bound.

We spell these identifications explicitly for clarity. Define a map $F:B\to \R^d$ by
\begin{align*}
F(x)=(\langle b^1,x\rangle,\dots,\langle b^d, x\rangle)~.
\end{align*}
We have an associated map $F^\star:B^\star\to \R^d$ given by
\begin{align*}
F^\star(x^\star) = (\langle x^\star,b_1\rangle,\dots,\langle x^\star,b_d\rangle)~.
\end{align*}
Since $\langle b^i,b_j\rangle = \delta_{ij}$, these maps respect the action of dual vectors in $B^\star$. That is,
\begin{align*}
\langle x, y\rangle = F^\star(x)\cdot F(y)~.
\end{align*}

Further, since each $\|b_i\|=\|b_i\|_\star=1$, we have 
\begin{align*}
\|F(x)\|^2=\sum_{i=1}^d \langle b^i,x\rangle^2\le d\|x\|^2~.
\end{align*}
and 
\begin{align*}
\|F^\star (x)\|^2=\sum_{i=1}^d \langle x,b_i\rangle^2\le d\|x\|_\star^2~.
\end{align*}
where the norm in $\R^d$ is the 2-norm. To make the correspondence notation cleaner, we write $\overline{x} = F(x)$ for $x\in B$ and $\overline{y} = F^\star(y)$ for $y\in B^\star$. $\overline{x}_i$ indicates the $i$th coordinate of $\overline{x}$.

Given any linear map $M:B\to B^\star$ (which we denote by $M\in \mathcal{L}(B,B^\star)$), there is an associated map $\overline{M}:\R^d\to\R^d$ given by
\begin{align*}
\overline{M} = F^\star M F^{-1}~.
\end{align*}
Further, when written as a matrix, the $ij$th element of $\overline{M}$ is
\begin{align*}
\overline{M}_{ij} = (F^\star MF^{-1} e_j)\cdot e_i,
\end{align*}
where $e_j$ represents the $j$th standard basis element in $\R^d$. A symmetric statement holds for any linear map $B^\star\to B$, in which $\overline{M}=F M (F^\star)^{-1}$.

These maps all commute properly: $\overline{Mx} = \overline{M}\overline{x}$ for any $M\in \mathcal{L}(B,B^\star)$ and $x\in B$, and similarly $\overline{Mx}=\overline{M}\overline{x}$ for any $M\in\mathcal{L}(B^\star, B)$ and $x\in B^\star$. It follows that $\overline{M}^{-1}=\overline{M^{-1}}$ for any $M$ as well.

Now, let's calculate $\overline{L}_{ij}$:
\[
\overline{L}_{ij} = (F^\star LF^{-1} e_j)\cdot e_i
=\langle L b_j,b_i\rangle
=\delta_{ij},
\]
so that the matrix $\overline{L}$ is the identity.

Finally, if $M_g:B\to B^\star$ is the map $M_g(x) = \langle g, x\rangle g$, then a simple calculation shows
\begin{align*}
\overline{M_g} =\overline{g}\overline{g}^T~.
\end{align*}

With these details described, recall that we are trying to bound the sum
\begin{align*}
\sum_{t=1}^T \langle z_t,  A^{-1}_{t}(z_t)\rangle~.
\end{align*}
We transfer to $\R^d$ coordinates:
\begin{align*}
\sum_{t=1}^T \langle z_t,  A^{-1}_{t}(z_t)\rangle=\sum_{t=1}^T  \overline{z_t}\cdot \overline{A_t}^{-1}\overline{z_t}~.
\end{align*}
We have $\|\overline{z_n}\|\le \sqrt{d}\|z_n\|_\star$ and
\begin{align*}
\overline{A_t}=\tau \overline{L} + \sum_{t=1}^t \overline{z_t}\overline{z_t}^T,
\end{align*}
so that by \citep{hazan2007logarithmic} Lemma 11,
\[
\sum_{t=1}^T  \overline{z_t}\cdot \overline{A_t}^{-1}\overline{z_t}
\le \ln \frac{|\overline{A_T}|}{|\overline{A_0}|} 
\leq d \ln \left(\frac{\sum_{t=1}^T \|\overline{z_t}\|^2}{d \tau} +1\right)
\leq d \ln \left(\frac{\sum_{t=1}^T \|z_t\|_\star ^2}{\tau} +1\right),
\]
where in the second inequality we used the fact that the determinant is maximized when all the eigenvalues are equal to $ \frac{\sum_{t=1}^T \|\overline{z_t}\|^2}{d}$.
\end{proof}

For completeness, we also state the regret bound and the setting of the parameters $\beta$ and $\tau$ to obtain a regret bound for exp-concave functions. Note that we use a different settings in Algorithms~\ref{algorithm:ons_1d} and \ref{algorithm:ons}, tailored to our specific setting.
\begin{theorem}\label{thm:finalonsbound}
Suppose we run Algorithm \ref{algorithm:ons} on $\alpha$ exp-concave losses. Let $D$ be the diameter of the domain $S$ and $\|\nabla f(x)\|_\star \leq Z$ for all the $x$ in $S$. Then set $\beta = \frac{1}{2}\min\left(\frac{1}{4ZD},\alpha\right)$ and $\tau=\frac{1}{\beta^2D^2}$. Then
\begin{align*}
R_T(\v) &\le 4d\left(ZD + \frac{1}{\alpha}\right)(1+\ln(T+1))~.
\end{align*}
\end{theorem}
\begin{proof}
First, observe that classic analysis of $\alpha$ exp-concave functions \citep[Lemma 3]{hazan2007logarithmic} shows that for any $x,y\in S$,
\begin{align*}
f(x)\ge f(y) + \langle \nabla f(y), x-y\rangle +\frac{\beta}{2}\langle \nabla f(y),x-y\rangle^2~.
\end{align*}
(Note that although the original proof is stated in $\R^d$, the exact same argument applies in a Banach space)

Therefore, by Theorems \ref{thm:onsfirstpass} and \ref{thm:useauerbach}, we have
\[
R_T(u) \le \frac{\beta\tau}{2}d\|u\|^2 +\frac{2}{\beta}d\ln(Z^2T/\tau + 1)~.
\]
Substitute our values for $\beta$ and $\tau$ to conclude
\[
R_T(u) 
\le \frac{d}{2\beta}\left(1+\ln(Z^2T\beta^2D^2+1)\right)
\le 4d\left(ZD + \frac{1}{\alpha}\right)(1+\ln(T+1)),
\]
where in the last line we used $\frac{1}{\beta}\le8(ZD + 1/\alpha)$.
\end{proof}

\section{Proofs of Theorems~\ref{thm:1dregret} and \ref{thm:banachregret}}
\label{sec:proof_regret_ons}

In order to prove Theorem~\ref{thm:1dregret} and \ref{thm:banachregret}, we first need some technical lemmas.
In particular, first we show in Lemma~\ref{thm:log_bound} that ONS gives us a logarithmic regret against the functions $\ell_t(\beta)=\ln(1+\langle g_t,\beta\rangle)$. Then, we will link the wealth to the regret with respect to an arbitrary unitary vector thanks to Theorem~\ref{thm:regretfixedunit}.

\begin{lemma}\label{thm:logbound}
\label{eq:upper_bound_log}
For $-1< x\leq2$, we have
\[
\ln(1+x)\leq x-\frac{2-\ln(3)}{4}x^2~.
\]
\end{lemma}

\begin{lemma}\label{thm:beta}
Define $\ell_t(v)=-\ln(1 - \langle g_t, v\rangle )$.
Let $\|\v\|, \|v\| \leq \frac{1}{2}$ and $\|g_t\|_\star  \leq 1$. Then
\begin{align*}
\ell_t(v)-\ell_t(\v)
\leq \langle \nabla \ell_t(v), v-\v \rangle -\frac{2-\ln(3)}{2}\frac{1}{2}\langle \nabla \ell_t(v),v-\v \rangle^2~.
\end{align*}
\end{lemma}
\begin{proof}
We have
\[
\ln(1-\langle g_t, \v\rangle ) 
= \ln(1-\langle g_t, v\rangle + \langle g_t, v-\v\rangle )
= \ln(1-\langle  g_t, v\rangle) + \ln\left(1 + \frac{\langle g_t,v-\v\rangle}{1-\langle g_t, v\rangle}\right)~.
\]
Now, observe that since $1-\langle g_t, \v\rangle \ge 0$ and $1-\langle g_t,v \rangle \ge 0$, $1 + \frac{\langle g_t,v-\v\rangle}{1-\langle g_t,v \rangle}\ge 0$ as well so that $\frac{\langle g_t, v-\v\rangle}{1-\langle g_t,v\rangle}\ge -1$. Further, since $\|\v-v\|\le 1$ and $1-\langle g_t, v\rangle\ge 1/2$, $\frac{\langle g_t,v-\v\rangle}{1-\langle g_t, v \rangle}\le 2$. Therefore, by Lemma \ref{thm:logbound} we have
\[
\ln(1-\langle g_t, \v\rangle)
\le \ln(1-\langle g_t, v\rangle) + \frac{\langle g_t, v-\v\rangle}{1-\langle g_t, v\rangle} - \frac{2-\ln(3)}{4} \frac{\langle g_t,v-\v\rangle^2}{(1- \langle g_t, v\rangle)^2}~.
\]
Using the fact that $\nabla \ell_t(v)=\frac{g_t}{1- \langle g_t, v\rangle}$ finishes the proof.
\end{proof}

\begin{lemma}\label{thm:log_bound}
Define $S=\{v \in B: \|v\|\leq\frac{1}{2}\}$ and $\ell_t(v):S\rightarrow \R$ as $\ell_t(v)=-\ln(1-\langle g_t, v\rangle)$, where $\|g_t\|_\star \leq 1$.
If we run ONS in Algorithm~\ref{algorithm:ons} with $\beta=\frac{2-\ln(3)}{2}$, $\tau=1$, and $S=\{v : \|v\|\leq\frac{1}{2}\}$, then
\[
\sum_{t=1}^T \ell_t(v_t)-\ell_t(\v)
\leq d\left(\frac{1}{17} + 4.5 \ln\left(1+4 \sum_{t=1}^T \|g_t\|_\star ^2\right)\right)~.
\]
\end{lemma}
\begin{proof}
From Lemma~\ref{thm:beta}, we have
\begin{align*}
\sum_{t=1}^T \ell_t(v_t)-\ell_t(\v)
\le\sum_{t=1}^T \left(\langle \nabla \ell_t(v_t), v_t - \v\rangle  -\frac{\beta}{2} \langle \nabla \ell_t(v_t),v_t -\v\rangle^2\right)~.
\end{align*}
So, using Lemma~\ref{thm:onsfirstpass} we have
\begin{align*}
\sum_{t=1}^T \left(\langle \nabla \ell_t(v_t), v_t - \v\rangle  -\frac{\beta}{2}\langle \nabla \ell_t(v_t),v_t -\v\rangle^2\right)
\le \frac{\beta}{2} \langle L(\v), \v\rangle + \frac{2}{\beta} \sum_{t=1}^T \langle z_t,  A^{-1}_{t}(z_t)\rangle,
\end{align*}
where $z_t=\nabla \ell_t(v_t)$. Now, use Theorem~\ref{thm:useauerbach} so that 
\begin{align*}
\frac{\beta}{2} \langle L(\v),\v\rangle + \frac{2}{\beta}\sum_{t=1}^T \langle z_t,  A^{-1}_{t}(z_t) \rangle
\le \frac{d\beta}{8} + \frac{2d}{\beta} \ln\left(1+\sum_{t=1}^T \|z_t\|^2_\star \right),
\end{align*}
where we have used $\|\v\|\le 1/2$.
Then observe that $\|z_t\|_\star ^2=\frac{\|g_t\|_\star ^2}{(1+ \langle g_t, \beta_t\rangle )^2} \le 4\|g_t\|_\star ^2$ so that $\ln(1+\sum_{t=1}^T \|z_t\|_\star ^2)\le \ln(1+4\sum_{t=1}^T \|g_t\|_\star ^2)$. Finally, substitute the specified value of $\beta$ and numerically evaluate to conclude the bound.
\end{proof}

Now, we collect some Fenchel conjugate calculations that allow us to convert our wealth lower-bounds into regret upper-bounds:
\begin{lemma}
\label{lemma:fenchel_exp}
Let $f(x)=a \exp(b |x|)$, where $a,b>0$. Then 
\[
f^\star (\theta) = 
\begin{cases}
\frac{|\theta|}{b} \left(\ln \frac{|\theta|}{a b}-1\right),  \quad \frac{|\theta|}{a b} > 1 \\
-a,  \quad \text{otherwise.}
\end{cases}
\leq \frac{|\theta|}{b} \left(\ln\frac{|\theta|}{a b}-1\right)~.
\]
\end{lemma}

\begin{lemma}
\label{lemma:fenchel_exp2}
Let $f(x)=a \exp(b \frac{x^2}{|x|+c})$, where $a,b>0$ and $c\geq0$. Then 
\[
f^\star (\theta) 
\leq |\theta| 
\max\left(\frac{2}{b} \left(\ln \frac{2|\theta|}{a b}-1\right), \sqrt{\frac{c}{b} \ln\left(\frac{c \theta^2}{a^2 b}+1\right)}-a\right)~.
\]
\end{lemma}
\begin{proof}
By definition we have
\[
f^\star (\theta)= \sup_{x} \ \theta x - f(x)~.
\]
It is easy to see that the sup cannot attained at infinity, hence we can safely assume that it is attained at $x^\star \in \R$.
We now do a case analysis, based on $x^\star $.

\textbf{Case $|x^\star |\leq c$.} In this case, we have that $f(x^\star ) \geq a \exp(b \frac{x^2}{2c})$, so
\begin{align*}
f^\star (\theta) 
&= \theta x^\star  - f(x^\star ) 
\leq \theta x^\star  - a \exp\left(b \frac{(x^\star )^2}{2c}\right) \\
&\leq \sup_x \ \theta x - a \exp\left(b \frac{x^2}{2c}\right) 
\leq |\theta| \sqrt{\frac{c}{b} \ln\left(\frac{c \theta^2}{a^2 b}+1\right)}-a,
\end{align*}
where the last inequality is from Lemma~18 in \cite{orabona2016coin}.

\textbf{Case $|x^\star | > c$.}
In this case, we have that $f(x^\star ) \geq a \exp\left(b \tfrac{(x^\star )^2}{2|x^\star |}\right) = a \exp\left(\tfrac{b}{2} |x^\star |\right)$, so
\begin{align*}
f^\star (\theta) 
&= \theta x^\star  - f(x^\star ) 
\leq \theta x^\star  - a \exp\left(\frac{b}{2} |x^\star |\right) \\
&\leq \sup_x \ \theta x - a \exp\left(\frac{b}{2} |x|\right)
\leq \frac{2|\theta|}{b} \left(\ln\frac{2|\theta|}{a b}-1\right),
\end{align*}
where the last inequality is from Lemma~\ref{lemma:fenchel_exp}.

Considering the max over the two cases gives the stated bound.
\end{proof}

\begin{theorem}\label{thm:wealthfixedunit}
Let $u$ be an arbitrary unit vector and $\|g_t\|_\star \leq 1$ for $t=1,\cdots,T$. Then
\begin{align*}
\sup_{\|v\|\le \frac{1}{2}} \ \sum_{t=1}^T \ln(1-\langle g_t,v\rangle)
\ge \frac{1}{4}\frac{\langle \sum_{t=1}^T g_t, u\rangle^2}{\sum_{t=1}^T \langle g_t,u\rangle^2+\left|\left \langle \sum_{t=1}^T g_t,u\right\rangle\right|}~.
\end{align*}
\end{theorem}
\begin{proof}
Recall that $\ln(1+x)\ge x-x^2$ for $|x|\le 1/2$. Then, we compute
\begin{align*}
\sup_{\|v\|\le 1/2} \ \sum_{t=1}^T \ln(1-\langle g_t,v\rangle)
&\ge \sup_{\|v\|\le 1/2} \ \sum_{t=1}^T \left(-\langle g_t,v\rangle - \langle g_t,v\rangle^2\right)\\
&= \sup_{\|v\|\le 1/2} \  -\left\langle \sum_{t=1}^T g_t, v\right\rangle - \sum_{t=1}^T \langle g_t,v\rangle^2~.
\end{align*}
Choose $v = \frac{u}{2}\frac{\left \langle \sum_{t=1}^T g_t,u\right\rangle }{\sum_{t=1}^T \langle g_t,u\rangle^2+\left|\left \langle \sum_{t=1}^T g_t,u\right\rangle\right|}$. Then, clearly $\|v\|\le\frac{1}{2}$. Thus, we have
\begin{align*}
\sup_{\|v\|\le 1/2} \ &\sum_{t=1}^T \ln(1+\langle g_t,v\rangle)
\ge \sup_{\|v\|\le 1/2} \  -\left\langle \sum_{t=1}^T g_t, v\right\rangle - \sum_{t=1}^T \langle g_t,v\rangle^2\\
&\ge \frac{1}{2}\frac{\langle \sum_{t=1}^T g_t, u\rangle^2}{\sum_{t=1}^T \langle g_t,u\rangle^2+\left|\left \langle \sum_{t=1}^T g_t,u\right\rangle\right|} - \frac{ \langle \sum_{t=1}^T g_t,u\rangle^2 }{4\left(\sum_{t=1}^T \langle g_t,u\rangle^2+\left|\left \langle \sum_{t=1}^T g_t,u\right\rangle\right|\right)^2}\sum_{t=1}^T\langle g_t,u\rangle^2\\
&\ge \frac{1}{4}\frac{\langle \sum_{t=1}^T g_t, u\rangle^2}{\sum_{t=1}^T \langle g_t,u\rangle^2+\left|\left \langle \sum_{t=1}^T g_t,u\right\rangle\right|}~.
\end{align*}
\end{proof}

\begin{lemma}\label{thm:regretfixedunit}
Let $u$ be an arbitrary unit vector in $B$ and $t>0$. Then, using the Algorithm~\ref{algorithm:ons}, we have
\begin{align*}
R_T(tu) 
&\leq \epsilon+t 
\max\left[\frac{d}{2}-8+8\ln\frac{8t\left(4\sum_{t=1}^T \|g_t\|^2_\star +1\right)^{4.5d}}{\epsilon},\right.\\\\
&\quad\quad\quad\quad\quad\quad\left.2\sqrt{\sum_{t=1}^T \langle g_t,u\rangle^2 \ln\left(\frac{5t^2}{\epsilon^2}\exp\left(\frac{d}{17}\right)\left(4\sum_{t=1}^T \|g_t\|^2+1\right)^{9d+1}+1\right)}\right]~.
\end{align*}
\end{lemma}
\begin{proof}
Let's compute a bound on our wealth, $\wealth_T$. We have that
\begin{align*}
\wealth_t 
&= \wealth_{t-1} - \langle g_t, w_t\rangle
= \wealth_{t-1}(1-\langle g_t, v_t\rangle)
= \epsilon \prod_{t=1}^T (1-\langle g_t, v_t\rangle),
\end{align*}
and taking the logarithm we have
\begin{align*}
\ln \wealth_t 
= \ln \epsilon + \sum_{t=1}^T \ln(1-\langle g_t, v_t\rangle)~.
\end{align*}
Hence, using Lemma~\ref{thm:log_bound}, we have
\[
\ln \wealth_t \geq \ln \epsilon + \max_{\|v\|\leq \frac{1}{2}} \sum_{t=1}^T \ln(1+\langle g_t, v\rangle)-d \left(\frac{1}{17}+4.5\ln\left(1+\sum_{t=1}^T 4\|g_t\|^2_\star \right)\right)~.
\]
Using Theorem~\ref{thm:wealthfixedunit}, we have
\begin{align*}
\wealth_T
&\ge \frac{\epsilon}{\exp\left[d \left(\frac{1}{17}+4.5\ln\left(1+4\sum_{t=1}^T \|g_t\|^2_\star \right)\right)\right]} \exp\left[\frac{1}{4}\frac{\langle \sum_{t=1}^T g_t, u\rangle^2}{\sum_{t=1}^T \langle g_t,u\rangle^2+\left|\left \langle \sum_{t=1}^T g_t,u\right\rangle\right|}\right]~.
\end{align*}
Defining
\begin{align*}
f(x) = \frac{\epsilon}{\exp\left[d \left(\frac{1}{17}+4.5\ln\left(1+4\sum_{t=1}^T \|g_t\|^2_\star \right)\right)\right]}\exp\left[\frac{1}{4}\frac{x^2}{\sum_{t=1}^T \langle g_t,u\rangle^2+|x|}\right],
\end{align*}
we have
\begin{align*}
&R_T(t u) 
= \epsilon-\wealth_T-t\left\langle \sum_{t=1}^T g_t, u\right\rangle\\
&\leq\epsilon-t\left\langle \sum_{t=1}^T g_t, u\right\rangle -f\left(\left\langle \sum_{t=1}^T g_t, u\right\rangle\right)\\
&\le \epsilon+f^\star (-t)\\
&\le \epsilon+t 
\max\left[8 \left(\ln\frac{8t}{\epsilon}+\frac{d}{17}+4.5 d\ln\left(4\sum_{t=1}^T \|g_t\|^2_\star +1\right)-1\right),\right.\\
&\quad\quad\quad\left.\sqrt{4\sum_{t=1}^T \langle g_t,u\rangle^2 \ln\left(\frac{5t^2}{\epsilon^2}\exp\left(\frac{d}{17}\right)\left(4\sum_{t=1}^T \|g_t\|^2+1\right)^{9d}\sum_{t=1}^T \langle g_t,u\rangle^2+1\right)}\right]\\
&\le \epsilon+t 
\max\left[\frac{d}{2}-8+8\ln\frac{8t\left(4\sum_{t=1}^T \|g_t\|^2_\star +1\right)^{4.5d}}{\epsilon},\right.\\
&\quad\quad\quad\left.2\sqrt{\sum_{t=1}^T \langle g_t,u\rangle^2 \ln\left(\frac{5t^2}{\epsilon^2}\exp\left(\frac{d}{17}\right)\left(4\sum_{t=1}^T \|g_t\|^2+1\right)^{9d+1}+1\right)}\right],
\end{align*}
where we have used the calculation of Fenchel conjugate of $f$ from Lemma~\ref{lemma:fenchel_exp2}. Then observe that $\exp(d/17)\le \exp((9d+1)/153)\le 2^{9d+1}$ to conclude:
\begin{align*}
R_T(t u)&\le \epsilon+t 
\max\left[\frac{d}{2}-8+8\ln\frac{8t\left(4\sum_{t=1}^T \|g_t\|^2_\star +1\right)^{4.5d}}{\epsilon},\right.\\
&\quad\quad\quad\left.2\sqrt{\sum_{t=1}^T \langle g_t,u\rangle^2 \ln\left(\frac{5t^2}{\epsilon^2}\left(8\sum_{t=1}^T \|g_t\|^2+2\right)^{9d+1}+1\right)}\right]~.
\end{align*}
\end{proof}

\begin{proof}[Proof of Theorem~\ref{thm:banachregret}]
Given some $\w$, set $u = \frac{\w}{\|\w\|}$ and $t=\|\w\|$. Then observe that $t^2\sum_{t=1}^T \langle g_t,u\rangle^2=\sum_{t=1}^T \langle g_t,\w\rangle^2$ and apply the previous Lemma~\ref{thm:regretfixedunit} to conclude the desired result.
\end{proof}

\section{Proof of Proposition \ref{prop:sd_convex} and Theorem~\ref{thm:fixedproj}}
\label{sec:sd_convex}
We restate Proposition \ref{prop:sd_convex} below:
\sdconvex*
\begin{proof}
Let $x,y\in B$, $t\in[0,1]$, $x'\in \Pi_W(x)$, and $y'\in \Pi_W(y)$.
Then
\begin{align*}
S_W(tx+(1-t)y)
&=\min_{d \in W} \|tx+(1-t)y - d\|
\le \|tx+(1-t)y - t x' - (1-t) y'\|\\
&=\|t(x-x') + (1-t)(y-y')\|
\le t\|x-x'\| + (1-t)\|y-y'\|\\
&=t S_W(x) + (1-t) S_W(y)~.
\end{align*}

For the Lipschitzness, let $x\in B$ and $x'\in \Pi_W(x)$, and observe that
\[
S_W(x+\delta)
= \inf_{d \in W} \|x+\delta - d\|
\le \|x+\delta - x'\|
\le S_W(x) +\|\delta\|~.
\]
Similarly, let $x\in B$, $\delta$ such that $x+\delta \in B$ and $x'\in \Pi_W(x+\delta)$, then
\[
S_W(x)
= \min_{d \in W}\|x - d\|
\le\|x+\delta-\delta - x'\|
\le S_W(x+\delta)+\|\delta\|~.
\]
So that $|S_W(x)-S_W(x+\delta)|\le \|\delta\|$.
\end{proof}

Now we restate and prove Theorem~\ref{thm:fixedproj}:
\fixedproj*

\begin{proof}
Let $x'=\frac{x+p}{2}$. Then clearly $S_W(x')\le \|x'-p\|=\frac{\|x-p\|}{2} = S_W(x)-\|x-x'\|$. Since $S_W$ is 1-Lipschitz, $S_W(x')\ge S_W(x)-\|x-x'\|$ and so $S_W(x')=S_W(x)-\|x-x'\|$. 

Suppose $g\in \partial S_W(x)$. Then $\langle g,x'-x\rangle +S_W(x)\le S_W(x')=S_W(x)-\|x-x'\|$. Therefore, $\langle g, x'-x\rangle \le -\|x-x'\|$. Since $\|g\|_\star \le 1$, we must have $\|g\|_\star =1$ and $\langle g,  x-p\rangle =\|x-p\|$. By assumption, this uniquely specifies the vector $(x-p)^\star$. Since $\partial S_W$ is not the empty set, $\{(x-p)^\star\} = \partial S_W(x)$.
\end{proof}

\section{Computing $S_W$ for multi-scale experts}\label{sec:simplexprojection}

In this section we show how to compute $\Pi_W(x)$ and a subgradient of $S_W(x)$ in Algorithm \ref{alg:multiscale}. First we tackle $\Pi_W(x)$. Without loss of generality, assume the $c_i$ are ordered so that $c_1\ge c_2\ge \cdots\ge c_N$. We also consider $W_k = \{x:x_i\ge 0\text{ for all }i\text{ and }\sum_{i=1}^N x_i/c_i=k\}$ instead of $W=W_1$. Obviously we are particularly interested in the case $k=1$, but working in this mild generality allows us to more easily state an algorithm for computing $\Pi_W(x)$ in a recursive manner.

\begin{Proposition}\label{thm:canbegreedy}
Let $N>1$ and $W_k=\{x:x_i\ge 0\text{ for all }i\text{ and }\sum_{i=1}^N x_i/c_i=k\}$, and let $S_{W_k}(x) = \inf_{y\in W_k} \|x-y\|_1$. Suppose the $c_i$ are ordered so that $c_1\ge c_2\ge\cdots\ge c_N$. Then for any $x=(x_1,\dots,x_n)$, there exists a $y=(y_1,\dots,y_n)\in \Pi_{W_k}(x)$ such that
\begin{align*}
y_1 = \left\{\begin{array}{lr}0,&x_1<0\\x_1,&x_1\in[0,kc_1]\\kc_1,&x_1>kc_1\end{array}\right.
\end{align*}
\end{Proposition}
\begin{proof}
First, suppose $N=1$. Then clearly there is only one element of $W_k$ and so the choice of $\Pi_{W_k}(x)$ is forced. So now assume $N>1$.

Let $(y_1,\dots,y_N)\in \Pi_{W_k}(x_1,\dots,x_N)$ be such that $|y_1-x_1|$ is as small as possible (such a point exists because $W_k$ is compact).

We consider three cases: either $x_1> kc_1,$ $x_1<0$ or $x_1\in [0, kc_1]$. 

\textbf{Case 1: $x> kc_1$}. Suppose $y_1 < kc_1$. Let $i$ be the largest index such that $y_i\ne 0$. $i\ne 1$ since $y_1/c_1<k$. Choose $0<\epsilon<\min(y_i\frac{c_1}{c_i}, kc_1-y_1)$. Then let $y'$ be such that $y'_1 = y_1+\epsilon$, $y'_i = y_i - \epsilon\frac{c_i}{c_1}$ and $y'_j = y_j$ otherwise. Then by definition of $\epsilon$, $y'_i\ge 0$ and $y'_1\le k c_1$. Further, $\sum_{j=1}^N y'_j/c_j = \epsilon/c_1 - \frac{c_i}{c_1}\epsilon/c_i+\sum_{j=1}^N y_j/c_j=k$ so that $y'\in W_k$. However, since $x_1>kc_1,$ $\|y'-x\|_1\le \|y-x\|_1-\epsilon + \epsilon\frac{c_i}{c_1}\le \|y-x\|_1$. Therefore, $y'\in \Pi_{W_k}(x)$, but $|y'_1-x_1|< |y_1-x_1|$, contradicting our choice of $y_1$. Therefore, $y_1=k c_1$.

\textbf{Case 2: $x<0$}. This case is very similar to the previous case. Suppose $y_1>0$. Let $i$ be the largest index such that $y_i\ne kc_i$. $i\ne 1$ since otherwise $\sum_{j=1}^N y_j/c_j > \sum_{j=2}^N k = k(N-1)\ge k$, which is not possible. Choose $0<\epsilon<\min(y_1,c_1(kc_i - y_i)/c_i)$. Set $y'$ such that $y'_1 = y_1-\epsilon$, $y'_i = y_i + \epsilon \frac{c_i}{c_1}$. Then, again we have $y'\in W_k$ and $\|y'-x\|_1\le \|y-x_1\|_1-\epsilon + \epsilon\frac{c_i}{c_1}\le \|y-x\|_1$ so that $y'\in \Pi_{W_k}(x)$, but $|y'_1-x_1|<|y_1-x_1|$. Therefore, we cannot have $y_1>0$ and so $y_1=0$.

\textbf{Case 3: $x\in[0,kc_1]$}. Suppose $y_1<x_1\le kc_1$. Then by the same the argument as for Case 1, there is some $i>1$ such that for any $0<\epsilon<\min(y_i\frac{c_1}{c_i}, x_1-y_1)$, we can construct $y'$ with $y'\in \Pi_{W_k}(x)$ and $|y'_1-x_1|<|y_1-x_1|$. Therefore, $y_1\ge x_1$.

Similarly, if $y_1>x_1$, then by the same argument as for Case 2, there is some $i>1$ such that for any $0<\epsilon<\min(y_1-x_1,c_1(kc_i - y_i)/c_i)$, we again construct $y'$ with $y'\in \Pi_{W_k}(x)$ and $|y'_1-x_1|<|y_1-x_1|$. Therefore, $y_1= x_1$.
\end{proof}

This result suggests an explicit algorithm for choosing $y\in \Pi_W(x)=\Pi_{W_1}(x)$. Using the Proposition we can pick $y_1$ such that there is a $y\in \Pi_{W_1}(x)$ with first coordinate $y_1$. If $y\in \Pi_{W_k}(x)$ has first coordinate $y_1$, then if $W^2_{k} = \{(y_2,\dots,y_n):y_i\ge 0\text{ for all }i\text{ and }\sum_{i=2}^N y_i/c_i=k\}$, then $(y_2,\dots,y_N)\in \Pi_{W^2_{k-y_1/c_1}}(x_2,\dots,x_N)$. Therefore, we can use a greedy algorithm to choose each $y_i$ in increasing order of $i$ and obtain a point $y\in\Pi_{W_k}(x)$ in $O(N)$ time. This procedure is formalized in Algorithm \ref{alg:piD}.

\begin{algorithm}[h]
   \caption{Computing $\Pi_W(x)$}
   \label{alg:piD}
\begin{algorithmic}[1]
   \REQUIRE $(x_1,\dots,x_N)\in\R^N$
   \STATE {\bfseries Initialize: } {$k_1=1$, $i=1$}
   \FOR{$i=1$ {\bfseries to} $N$}
   \IF{$i=N$}
        \STATE Set $y_i = k_i c_i$
   \ELSE
    \IF{$x_i\le 0$}
        \STATE Set $y_i = 0$
    \ENDIF
    \IF{$x_i>k_ic_i$}
        \STATE Set $y_i = k_ic_i$
    \ENDIF
    \IF{$x_i\in(0,k_ic_i]$}
        \STATE Set $y_i = x_i$
    \ENDIF
   \STATE Set $k_{i+1} = k_i-y_i/c_i$
   \ENDIF
   \ENDFOR
   \RETURN $(y_1,\dots,y_N)$
\end{algorithmic}
\end{algorithm}

\subsection{Computing a subgradient of $S_W$ for multi-scale experts}
Unfortunately, $\|\cdot\|_1$ does not satisfy the hypotheses of Theorem \ref{thm:fixedproj} and so we need to do a little more work to compute a subgradient.
\begin{Proposition}\label{thm:sdgradient}
Let $(y_1,\dots,y_n)$ be the output of Algorithm~\ref{alg:piD} on input $x=(x_1,\dots,x_N)$. Then if $i=N$, $\frac{\partial S_W(x)}{\partial x_i} = \sign(x_N-y_N)$. Let $M$ be the smallest index such that $y_M=k_Mc_M$, where $k_i$ is defined in Algorithm~\ref{alg:piD}. There exists a subgradient $g\in \partial S_W(x)$ such that
\begin{align*}
g_i=\left\{\begin{array}{lr}
-1, & x_i \le 0\\
1, & x_i> k_ic_i\\
\sign(x_M-y_M)\frac{c_M}{c_i}, & x_i\in(0,k_ic_i],\ x_M\ne k_Mc_M\\
\frac{c_M}{c_i}, & x_i\in(0,k_ic_i],\ x_M=k_Mc_M
\end{array}\right.
\end{align*}
\end{Proposition}
\begin{proof}
We start with a few reductions. First, we show that by a small perturbation argument we can assume $x_M\ne k_Mc_M$. Next, we show that it suffices to prove that $S_W$ is linear on a small $L_\infty$ ball near $x$. Then we go about proving the Proposition for that $L_\infty$ ball, which is the meat of the argument.

Before we start the perturbation argument, we need a couple observations about $M$. First, observe that $k_i=y_i=0$ for all $i>M$.

Next, we show that either have $M=N$, or $x_M\ge k_Mc_M$. If $M\ne N$, then by inspection of the Algorithm \ref{alg:piD}, we must have $x_M\le 0$ and $k_M=0$ or $x_M\ge k_Mc_M$. If $k_M=0$, then we have $0=k_{M} = k_{M-1}-\frac{y_{M-1}}{c_{M-1}}$. This implies $k_{M-1}c_{M-1}=y_{M-1}$, which contradicts our choice of $M$ as the smallest index with $y_M=k_Mc_M$. Therefore, we must have $x_M\ge k_Mc_M$. Therefore, we must have $M=N$, or $x_M\ge k_Mc_M$.

Now, we show that we may assume $x_M\ne k_Mc_M$. Let $\delta>0$. If $x_M\ne k_Mc_M$, set $x_\delta=x$. Otherwise, set $x_\delta=x+\delta e_M$. By inspecting Algorithm \ref{alg:piD}, we observe that the output on $x_\delta$ is unchanged from the output on $x$, and $M$ is still the smallest index such that $y_i=k_ic_i$.

We claim that it suffices to prove $g\in \partial S_W(x_\delta)$ for all $\delta$ rather than $g\in \partial S_W(x)$. To see this, observe that by 1-Lipschitzness, $|S_W(x_\delta)-S_W(x)|\le \delta$, so that if $g\in\partial S_W(x_\delta)$, then for any $w$,
\begin{align*}
    S_W(w)\ge S_W(x_\delta)+\langle g,w-x_\delta\rangle \ge S_W(x) + \langle g, w-x\rangle - 2\delta~.
\end{align*}
By taking $\delta\to 0$, we see that $g$ must be a subgradient of $S_W$ at $x$ if $g\in \partial S_W(x_\delta)$ for all $\delta$. This implies that if we prove the Proposition for any $x_\delta$, which has $x_M\ne k_Mc_M$, we have proved the proposition for $x$.

Following this perturbation argument, for the rest of the proof we consider only the case $x_M\ne k_Mc_M$.

Now, we claim that to show the Proposition, it suffices to exhibit a closed $L_\infty$ ball $B$ such that $x$ is on the boundary of $B$ and for $z\in B$, $S_W(z) = \langle g, z\rangle+F$ for some constant $F$. To see this, first suppose that we have such a $B$. Then observe that $g$ is the derivative, and therefore a subgradient, of $S_W$ for any point in the interior of $B$. Let $z$ be in the interior of $B$ and let $w$ be an arbitrary point in $\R^N$. Then since $g$ is a subgradient at $z$, we have $S_W(w)\ge S_W(z) + \langle g,w-z\rangle$. Further, since $x$ is on the boundary of $B$ (and therefore in $B$), $S_W(x) = S_W(z) + \langle g,x-z\rangle$. Putting these identities together:
\begin{align*}
S_W(w)&\ge S_W(z) + \langle g, w-z\rangle\\
&=S_W(z) + \langle g,x-z\rangle+\langle g,w-x\rangle\\
&= S_W(x) + \langle g,w-x\rangle~.
\end{align*}
Therefore, $g$ is a subgradient of $S_W$ at $x$.

Next, we turn to identifying the particular $L_\infty$ ball we will work with. Let 
\begin{align*}
    q &= \frac{1}{2}\min_{x_i>0} \ x_i,\\
    d&=\frac{1}{2}\min_{j|x_j\ne k_jc_j}\ \min(1/c_1,1)|x_j-c_jk_j|,\\
    h&=\min(q,d)\min(c_N,1)/N~.
\end{align*}
Consider the $L_\infty$ ball given by
\begin{align*}
B=\left\{x+(\epsilon_1,\dots,\epsilon_N)|\ \epsilon_j\in[-h,0]\right\}~.
\end{align*}


Clearly, $x$ is on the boundary of $B$. Now, we proceed to show that $S_W$ is linear on the interior of $B$, which will prove the Proposition by the above discussion.
\vskip2em
Let $x'=x+\epsilon$ be an element of $B$. We will compute $S_W(x')$ by computing the output $y'$ of running Algorithm \ref{alg:piD} on $x'$. We will also refer to the internally generated variables $k_i$ as $k'_i$ to distinguish between the $k$s generated when computing $y$ versus when computing $y'$. The overall strategy is to show that all of the conditional branches in Algorithm \ref{alg:piD} will evaluate to the same branch on $x$ as on $x'$.

Specifically we show the following claim by induction:
\begin{Claim}
for any $i<M$:
\begin{align*}
y'_i &= \left\{\begin{array}{lr}
0&x_i\le 0\\
x'_i&x_i\in (0,k_ic_i]
\end{array}\right.,\\
k'_{i+1}&= k_{i+1}+\sum_{j\le i,\ x_j\in(0,k_jc_j]} -\epsilon_j/c_j,\\
k_{i+1}&\le k'_{i+1}\le k_{i+1} + d\frac{i}{2N},\\
|y'_i-x'_i|&=\left\{\begin{array}{lr}
|y_i-x_i|-\epsilon_i&x_i\le 0\\
|y_i-x_i|&x_i\in (0,k_ic_i]
\end{array}
\right.~.
\end{align*}
For $i=M$,
\begin{align*}
y'_i &= k'_ic_i,\\
k'_{i+1}&=0,\\
|y'_i-x'_i| &= |y_i-x_i|+\sign(x_i-y_i)\epsilon_M+\sum_{j<M|\ x_j\in(0,k_jc_j]} c_M\epsilon_j/c_j~.
\end{align*}
And for $i>M$:
\begin{align*}
y'_i &= 0,\\
k'_{i+1}&=0,\\
|y'_i-x'_i| &= \left\{\begin{array}{lr}
|y_i-x_i|-\epsilon_i&x_i\le 0\\
|y_i-x_i|+\epsilon_i& x_i>0
\end{array}
\right.~.
\end{align*}
\end{Claim}

First we do the base case. Observe that $k'_1=k_1$. Then we consider three cases, either $x_1\le 0$, $x_1\in (0,k_1c_1]$, or $x_1> k_1c_1$. These cases correspond to $y_1=0$, $y_1=x_1$, or $y_1=k_1c_1$.

\textbf{Case 1 ($x_1\le 0$):} Since $\epsilon_1\le 0$, we have $x'_1=x_1+\epsilon_1\le 0$. Therefore, by inspecting the condition blocks in Algorithm \ref{alg:piD}, $y'_1 = y_1 = 0$ and $k'_2=k_2$.

\textbf{Case 2 ($x_1\in (0,k_1,c_1]$):} Since $x_1>0$, we have $|\epsilon_1|\le q\le x_1/2$. Therefore, $x'_1>0$. Since $\epsilon_1\le 0$, $x'_1\le x_1\le k_1c_1=k'_1c_1$ so that $x'_1\in(0,k'_1c_1]$. This implies $y'_1= x'_1$ and 
\begin{align*}
k'_2&=k'_1-\frac{x'_1}{c_1}\\
&=k_1-\frac{x_1+\epsilon_1}{c_1}\\
&=k_2-\frac{\epsilon_1}{c_1}~.
\end{align*}

\textbf{Case 3 ($x_1 > k_1c_1$):} In this last case, observe that $|\epsilon_1|<d\le (x_1-k_1c_1)/2$ so that $x_1\ge x'_1> k_1c_1=k'_1c_1$. This implies $y'_1 = k'_1c_1=k_1c_1$ and $k'_2=0$.

The values for $|y'_1-x'_1|$ can also be checked via the casework. First, suppose $1=M$. Then we must have $x_1>k_1c_1$ (because we assume $x_M\ne k_Mc_M$ by our perturbation argument). Therefore, $y_1=y'_1=k_1c_1$ and the base case is true. 

When $1<M$, then we consider the cases $x_1\le 0$ and $x_1\in(0,k_1c_1]$. The case $x_1>k_1c_1$ does not occur because $1<M$. When $x_1\le 0$, then by the above casework we must have $x'_1\le 0$ and $y'_1=y_1=0$. Therefore,
\begin{align*}
|y'_1-x'_1|=|x'_1|=|x_1| + |\epsilon_1| = |y_1-x_1|-\epsilon_1,
\end{align*}
where we have used $\epsilon_1 \le 0$ to conclude $|x'_1|= |x_1|+|\epsilon_1|$.

When $x_1\in(0,k_1c_1]$, we have $y_1=x_1$, and by the above casework we have  and $y'_1=x'_1$. Thus $|y'_1-x'_1|=0=|y_1-x_1|$. This concludes the base case of the induction.
\vskip2em

Now, we move on to the inductive step. Suppose the claim holds for all $j<i$. To show the claim also holds for $i$, we consider the three cases $i<M$, $i=M$ and $i>M$ separately:

\textbf{Case 1 ($i<M$):}
We must consider two sub-cases, either $x_i\le 0$, or $x_i\in (0,k_ic_i]$. The case $x_i> k_ic_i$ does not occur because $i<M$.

\textbf{Case 1a ($x_i\le 0$):}
In this case, we have $y_i=0$ and $k_{i+1}=k_i$. By definition, $\epsilon_i\le 0$ so that $x'_i\le 0$. Then by inspection of Algorithm \ref{alg:piD}, $y'_i=0=y_i$ so that $k'_{i+1}=k'_i$. By the induction assumption, this implies 
\begin{align*}
k'_{i+1}=k'_i =k_i+ \sum_{j< i,\ x_j\in(0,k_jc_j]} -\epsilon_j/c_j =k_{i+1}+\sum_{j\le i,\ x_j\in(0,k_jc_j]} -\epsilon_j/c_j~.
\end{align*}
Also, $k'_{i+1}=k'_i\ge k_i=k_{i+1}$ and also
\begin{align*}
|k'_{i+1}-k_{i+1}|&=|k'_i - k_i|\le d\frac{i-1}{N}\le d\frac{i}{N}~.
\end{align*}
Finally, since $y'_i=0=y_i$ and $x_i,x'_i\le 0$, we have 
\begin{align*}
|y'_i-x'_i|=|x'_i| = -x'_i=-x_i-\epsilon_i = |x_i|-\epsilon_i = |y_i-x_i|-\epsilon_i~.
\end{align*}
Thus all parts of the claim continue to hold.

\textbf{Case 1b ($x_i\in(0,k_ic_i]$):} In this case we show that $x'_i\in (0, k'_i,c_i]$. Observe that $y_i=x_i$ and $k_{i+1}=k_i-x_i/c_i$. By definition again, $\epsilon_i\le 0$, and also $|\epsilon_i|\le q\le x_i/2$, so that $x'_i> 0$. Finally, since $k'_i\ge k_i$,
\begin{align*}
x'_i\le x_i\le c_ik_i\le c_ik'_i~.
\end{align*}
Therefore, $x'_i\in (0,k'_ic_i]$ so that $y'_i=x'_i$ and
\begin{align*}
k'_{i+1}&=k'_i-x'_i/c_i\\
&=k_i + (k'_i-k_i) - x_i/c_i - \epsilon_i/c_i\\
&=k_{i+1}+(k'_i-k_i) - \epsilon_i/c_i\\
&=k_{i+1}+\sum_{j\le i,\ x_j\in(0,k_jc_j]} -\epsilon_j/c_j,
\end{align*}
where the last equality uses the induction assumption. Now, since $\epsilon_j\le 0$ for all $j$, this implies $k'_{i+1}\ge k_{i+1}$. Further, $|\epsilon_i/c_i|\le d c_N/(N c_i)\le d/N$ and by the inductive assumption, $|k'_i-k_i|\le d\frac{i-1}{N}$ so that $|k'_{i+1}-k_{i+1}|\le d\frac{i}{N}$ as desired. Finally, since $y'_i=x'_i$ and $y_i=x_i$, $|y'_i-x'_i|=0=|y_i-x_i|$.

\vskip1em

\textbf{Case 2 ($i=M$):} First we show that $y'_i=k'_i c_i$, which implies $k'_{i+1}=0$, and then we prove the expression for $|y'_i-x'_i|$. Since $x_M\ne k_Mc_M$, we must have either either $x_i> k_ic_i$ or $M=N$.

If $M=N$, then the claim $y'_i=k'_ic_i$ is immediate by inspection of Algorithm \ref{alg:piD}. So suppose $x_i>k_ic_i$. By the inductive assumption, $k'_i\le k_i + d\frac{i}{N}\le k_i+d$. Now, we observe that $d\le \frac{1}{2c_1}(x_i-c_ik_i)\le \frac{1}{2c_i}(x_i-c_ik_i)$, which implies
\begin{align*}
c_ik'_i&\le c_ik_i + c_i d\\
&\le c_ik_i + (x_i-c_ik_i)/2\\
&\le x_i-(x_i-c_ik_i)/2~.
\end{align*}
Next, observe that $d\le \frac{1}{2}(x_i-c_ik_i)$ to conclude
\begin{align*}
c_ik'_i&\le x_i-(x_i-c_ik_i)/2\\
&\le x_i-d\\
&\le x_i-h\\
&\le x'_i~.
\end{align*}
Therefore, $x'_i \ge k'_ic_i$, so that $y'_i=c_ik'_i$.

It remains to compute $|y'_i-x'_i|$. By the induction assumption, we have 
\begin{align*}
k'_i=k_i+\sum_{j< i,\ x_j\in(0,k_jc_j]} -\epsilon_j/c_j~.
\end{align*}
Therefore, 
\begin{align}
x'_i-y'_i&=x_i+\epsilon_M-y_i +c_M\sum_{j< i,\ x_j\in(0,k_jc_j]}\epsilon_j/c_j~.\label{eqn:diff}
\end{align}
Observe that $\epsilon_M+c_M\sum_{j< i,\ x_j\in(0,k_jc_j]}\epsilon_j/c_j\le  0$ since $\epsilon_i\le 0$ for all $i\le M$. Now, since $c_M\le c_j$ for $j\le M$, we have
\begin{align*}
\left|\epsilon_M+c_M\sum_{j< i,\ x_j\in(0,k_jc_j]}\epsilon_j/c_j\right|\le Nh\le d~.
\end{align*}
Now, since $x_M\ne x_Mk_M$, and $i=M$, we have $d\le \frac{|x_i-c_ik_i|}{2}$ by definition so that 
\begin{align*}
\left|\epsilon_M+c_M\sum_{j< i,\ x_j\in(0,k_jc_j]}\epsilon_j/c_j\right|\le |x_i-c_ik_i|/2
=\frac{|x_i-y_i|}{2}~.
\end{align*} 
Now, recalling equation (\ref{eqn:diff}) we have
\begin{align*}
\sign(x'_i-y'_i)&=\sign\left(x_i-y_i +\left[\epsilon_M+c_M\sum_{j< i,\ x_j\in(0,k_jc_j]}\epsilon_j/c_j\right]\right)\\
&=\sign(x_i-y_i),
\end{align*}
where in the last line we have used $\left|\epsilon_M+c_M\sum_{j< i,\ x_j\in(0,k_jc_j]}\epsilon_j/c_j\right|\le \frac{|x_i-y_i|}{2}$. Therefore, we have
\begin{align*}
|x'_i-y'_i|&=\sign(x'_i-y'_i)(x'_i-y'_i)\\
&=\sign(x_i-y_i)\left(x_i-y_i + \epsilon_M+c_M\sum_{j< i,\ x_j\in(0,k_jc_j]}\epsilon_j/c_j\right)\\
&=|x_i-y_i| + \sign(x_i-y_i)\left(\epsilon_M+c_M\sum_{j< i,\ x_j\in(0,k_jc_j]}\epsilon_j/c_j\right)~.
\end{align*}

\vskip1em
\textbf{Case 3 ($i>M$):}

Since $k'_i=0$ by inductive hypothesis, we must have $y'_i=0$ as desired. Further, observe that as observed in the beginning of the proof, $k_i=0$ for all $i>M$ as well so that we have $y_i=0$. Finally, if $x_i>0$, we have $x_i+\epsilon_i\ge x_i/2>0$ since $|\epsilon_i|\le q\le x_i/2$ so that $\sign(x'_i)=\sign(x_i)$. Therefore, we can conclude
\begin{align*}
|y'_i-x'_i|=|x'_i|=\left\{\begin{array}{lr}
|x_i|-\epsilon_i&x_i\le 0\\
|x_i|+\epsilon_i& x_i>0
\end{array}\right.~.
\end{align*}
Since $y_i=0$, $|x_i|=|y_i-x_i|$ and this is the desired form for $|y'_i-x'_i|$.

This concludes the induction.

\vskip3em
From the expression for $|y'_i-x'_i|$ we see that if $g$ is given by
\begin{align*}
g_i=\left\{\begin{array}{lr}
-1&x_i \le 0\\
1&x_i> k_ic_i\\
\sign(x_M-y_M)\frac{c_M}{c_i}&x_i\in(0,k_ic_i],\ x_M\ne k_Mc_M\\
\frac{c_M}{c_i}&x_i\in(0,k_ic_i],\ x_M=k_Mc_M
\end{array}\right.
\end{align*}
then $S_W(x+\epsilon) = S_W(x)+\langle g,\epsilon\rangle$. Finally, observe that our perturbation $x_\delta$ has the property $\sign((x_\delta)_M-y_M) = 1$ if $x_M=k_My_M$ to prove the Proposition.
\end{proof}

\section{Proof of Theorem \ref{thm:metagrad}}\label{sec:metagradproof}
We re-state Theorem \ref{thm:metagrad} below for reference:
\metagrad*
\begin{proof}
For any $t$, consider the random vector $X_t$ that takes value $x_{i}$ for $i \le t$ with probability proportional to $\|\tilde{g}_{i}\|_\star ^2$ and value $\overline{x}_0$ with probability proportional to 1. Make the following definitions/observations:
\begin{enumerate}
\item $Z_t := 1 + \sum_{i=1}^t \|\tilde{g}_i\|_\star ^2$ for all $t$, so that 
\begin{align*}
V_T(\w)
= \|\overline{x}_0-\w\|^2 +  \sum_{t=1}^T \|\tilde{g}_t\|_\star^2\|x_t-\w\|^2
= Z_T \E[\|X_T-\w\|^2]~.
\end{align*}
\item $\overline{x}_T = \E[X_T] = \frac{\overline{x}_0+\sum_{t=1}^T \|\tilde{g}_t\|_\star ^2 x_t}{1 + \sum_{t=1}^T \|\tilde g_t\|_\star ^2}$.
\item $\sigma^2_t := \frac{\|\overline{x}_t-\overline{x}_0\|^2+\sum_{i=1}^t\|\tilde{g}_{i}\|^2_\star  \| x_{i}-\overline{x}_t\|^2}{Z_t}$ so that $\sigma^2_t = \E[\|X_t-\overline{x}_t\|^2]$, and $\sigma^2_T Z_T = \|\overline{x}_0-\overline{x}_T\|^2+\sum_{t=1}^T\|\tilde g_t\|_\star ^2 \|x_t-\overline{x}_T\|^2$.
\end{enumerate}

To prove the theorem, we are going to show for any $\w\in W$,
\begin{equation}
\label{eqn:biasvariancetarget}
R_T(\w) 
\le O\left[\sqrt{Z_T\|\w-\overline{x}_T\|^2\ln \frac{TD}{\epsilon^2}} + \ln\frac{D T}{\epsilon}\ln(T) + \sqrt{Z_T\sigma_T^2 \ln \frac{TD}{\epsilon}\log(T)}\right],
\end{equation}
which implies the desired bound by a bias-variance decomposition: $Z_T\|\w-\overline{x}_T\|^2 + Z_T\sigma^2_T = Z_T\E[\|X_T-\w\|^2]=  V_T(\w)$.

Observe that, by triangle inequality and the definition of dual norm, $\langle g_t,z\rangle +\|g_t\|_\star S_W(z) \ge \langle g_t, x\rangle$ for all $z$ and $x \in \Pi_W(z)$, with equality when $z\in W$. Hence, we have
\begin{equation}
\label{eq:metagrad_loss_bound}
\langle g_t,x_t-\w\rangle 
\le \langle g_t,z_t - \w\rangle + \|g_t\|_\star S_W(z_t)-\|g_t\|_\star S_W(\w)
\le \langle \tilde{g}_t, z_t-\w\rangle,
\end{equation}
for all $\w\in W$, where in the last inequality we used Proposition~\ref{prop:sd_convex}.
Using this inequality with the regret guarantee of $\ol$, we have
\begin{align*}
&R_T(\w)
\le\sum_{t=1}^T \langle g_t, x_t-\w\rangle
\leq \sum_{t=1}^T \langle \tilde g_t, z_t-\w\rangle
=\sum_{t=1}^T \langle \tilde{g}_t, w_t - (\w-\overline{x}_T)\rangle + \sum_{t=1}^T \langle \tilde{g}_t, \overline{x}_{t-1} - \overline{x}_T\rangle \\
&\leq O\left(\|\w-\overline{x}_T\|\sqrt{\sum_{t=1}^T \|\tilde{g}_t\|_\star ^2 \ln\frac{\|\w-\overline{x}_T\|T}{\epsilon^2}} + \|\w-\overline{x}_T\| \ln\frac{\|\w-\overline{x}_T\| T}{\epsilon}\right) + \epsilon+\sum_{t=1}^T \langle \tilde{g}_t, \overline{x}_{t-1} - \overline{x}_T\rangle \\
&=O\left(\sqrt{Z_T \|\w-\overline{x}_T\|^2 \ln\frac{DT}{\epsilon^2}} + D \ln\frac{D T}{\epsilon}\right) +  \epsilon + \sum_{t=1}^T \langle \tilde{g}_t, \overline{x}_{t-1} - \overline{x}_T\rangle~.
\end{align*}
Note that the first term is exactly what we want, so we only have to upper bound the second one. This is readily done through Lemma~\ref{lemma:shortcut} that immediately gives us the stated result.
\end{proof}

\begin{lemma}
\label{lemma:shortcut}
Under the hypotheses of Theorem~\ref{thm:metagrad}, we have 
\[
\sum_{t=1}^T \langle \tilde{g}_t, \overline{x}_{t-1} - \overline{x}_T\rangle
\leq  M \sqrt{Z_T} \sigma_T \sqrt{1+\ln Z_T}+K(1+\ln Z_T),
\]
where $M = A\sqrt{1+\ln\left(\frac{2D^2T^C}{\epsilon^2}+3T^C\right)}$ and $K=1+B\ln \left(\frac{\sum_{t=1}^T \|g_t\|_\star DT^C }{\epsilon}+2T^C\right)$.
\end{lemma}
\begin{proof}
We have that
\begin{align*}
\sum_{i=1}^t \langle \tilde{g}_i, \overline{x}_{i-1}-\overline{x}_t\rangle - \sum_{i=1}^{t-1} \langle \tilde{g}_i, \overline{x}_{i-1}-\overline{x}_{t-1}\rangle
&= \left\langle \sum_{i=1}^t \tilde{g}_{i},\overline{x}_{t-1}-\overline{x}_t\right\rangle~.
\end{align*}
The telescoping sum gives us
\[
\sum_{t=1}^T \langle \tilde{g}_t, \overline{x}_{t-1}-\overline{x}_T\rangle 
=\sum_{t=1}^{T} \left\langle \sum_{i=1}^t \tilde{g}_i,\overline{x}_{t-1}-\overline{x}_t\right\rangle
\leq \sum_{t=1}^{T} \left\| \sum_{i=1}^t \tilde{g}_i\right\|_\star  \|\overline{x}_{t-1}-\overline{x}_t\|~.
\]

So in order to bound $\sum_{t=1}^T \langle \tilde{g}_t,\overline{x}_{t-1}-\overline{x}_T\rangle$, it suffices to bound $\left\| \sum_{i=1}^t \tilde{g}_i\right\|_\star  \|\overline{x}_{t-1}-\overline{x}_t\|$ by a sufficiently small value.
First we will tackle $\left\|\sum_{i=1}^t \tilde{g}_i\right\|$. To do this we recall our regret bound for $\ol$. Analogous to \eqref{eq:metagrad_loss_bound}, we have
\begin{align*}
\langle g_t,x_t\rangle &\ge \langle g_t,z_t\rangle + \|g_t\|_\star S_W(z_t) + \langle \tilde{g}_t,x_t-z_t\rangle \\
\langle \tilde{g}_t, z_t\rangle&\ge \langle g_t,z_t-x_t\rangle  + \|g_t\|_\star  \|z_t-x_t\| + \langle \tilde{g}_t,x_t\rangle\\
&\ge \langle \tilde{g}_t,x_t\rangle~.
\end{align*}

Therefore, for any $X \in \R$ we have:
\begin{align*}
\sum_{i=1}^t &-\|\tilde{g}_i\|_\star D+\left\|\sum_{i=1}^t\tilde g_i\right\|_\star X \\
&\le \sum_{i=1}^t \langle \tilde{g}_i,x_i-\overline{x}_{i-1}\rangle +\left\|\sum_{i=1}^t\tilde g_i\right\|_\star X\\
&\le \sum_{i=1}^t \langle \tilde{g}_i,z_i-\overline{x}_{i-1}\rangle +\left\|\sum_{i=1}^t\tilde g_i\right\|_\star X\\
&= \sum_{i=1}^t \langle \tilde{g}_i,w_i\rangle+\left\|\sum_{i=1}^t\tilde g_i\right\|_\star X\\
&\le \epsilon + |X|A\sqrt{\sum_{i=1}^t \|\tilde{g}_i\|_\star^2\left(1+\ln\left(\frac{|X|^2t^C}{\epsilon^2}+1\right)\right)} + B|X|\ln\left(\frac{|X|t^C}{\epsilon}+1\right),
\end{align*}
where in the first inequality we have used the fact that the domain is bounded.

Dividing by $X$ and solving for $\left\|\sum_{i=1}^t\tilde g_i\right\|_\star$, we have
\begin{align*}
\left\|\sum_{i=1}^t\tilde g_i\right\|_\star 
&\le \frac{\epsilon}{X} + A\sqrt{\sum_{i=1}^t \|\tilde{g}_i\|_\star^2\left(1+\ln\left(\frac{|X|^2t^C}{\epsilon^2}+1\right)\right)} + B\ln\left(\frac{|X|t^C}{\epsilon}+1\right) +\frac{\sum_{i=1}^t \|\tilde{g}_i\|_\star D}{X}~.
\end{align*}
Set $X=\epsilon +\sum_{i=1}^t \|\tilde{g}_i\|_\star D$ and over-approximate to conclude:
\begin{align*}
\left\|\sum_{i=1}^t\tilde g_i\right\|_\star 
&\le 1 + A\sqrt{\sum_{i=1}^t \|\tilde{g}_i\|_\star^2\left(1+\ln\left(\frac{2D^2\left(\sum_{i=1}^t\|\tilde{g}_i\|_\star\right)^2t^C}{\epsilon^2}+3t^C\right)\right)} \\
&\quad + B\ln\left(\frac{\sum_{i=1}^t \|\tilde{g}_i\|_\star Dt^C}{\epsilon}+2t^C\right)\\
&\le M\sqrt{\sum_{i=1}^t\|\tilde{g}_i\|_\star^2} + K~.
\end{align*}

With this in hand, we have
\begin{equation}
\sum_{t=1}^T \langle \tilde{g}_t, \overline{x}_{t-1}-\overline{x}_T\rangle 
\le \sum_{t=1}^{T} \left\|\sum_{i=1}^t \tilde{g}_i\right\|_\star \|\overline{x}_{t-1}-\overline{x}_t\|
\le M\sum_{t=1}^{T}\sqrt{\sum_{i=1}^t \|\tilde{g}_i\|^2_\star }\|\overline{x}_{t-1}-\overline{x}_t\| + K \sum_{t=1}^{T}\|\overline{x}_{t-1}-\overline{x}_t\|~.\label{eqn:telescoped}
\end{equation}

Now, we relate $\|\overline{x}_t-\overline{x}_{t-1}\|$ to $\|x_t-\overline{x}_t\|$:
\[
\overline{x}_{t-1}-\overline{x}_t 
= \overline{x}_{t-1} - \frac{Z_{t-1}\overline{x}_{t-1} + \|\tilde{g}_t\|_\star ^2 x_t}{Z_t}
=\frac{\|\tilde{g}_t\|_\star ^2}{Z_t}(\overline{x}_{t-1} - x_t)
=\frac{\|\tilde{g}_t\|_\star ^2}{Z_t}(\overline{x}_{t} - x_t) + \frac{\|\tilde{g}_t\|_\star ^2}{Z_t}(\overline{x}_{t-1}-\overline{x}_t),
\]
that implies
\[
Z_t (\overline{x}_{t-1}-\overline{x}_t)
= \|\tilde{g}_t\|_\star ^2(x_t-\overline{x}_t)+\|\tilde{g}_t\|_\star ^2(\overline{x}_{t-1}-\overline{x}_t),
\]
that is
\begin{equation}
\label{eq:lemma_shortcut_eq1}
\overline{x}_{t-1}-\overline{x}_t
= \frac{\|\tilde{g}_t\|_\star ^2}{Z_{t-1}}(x_t-\overline{x}_t)~.
\end{equation}
Hence, we have 
\[
M\sum_{t=1}^{T}\sqrt{\sum_{i=1}^t \|\tilde{g}_i\|^2_\star }\|\overline{x}_{t}-\overline{x}_{t-1}\|
\leq M\sum_{t=1}^T \sqrt{Z_t}\frac{\|g_t\|_\star ^2}{Z_{t-1}}\|x_t-\overline{x}_t\|,
\]
and
\[
K\sum_{t=1}^{T}\|\overline{x}_{t}-\overline{x}_{t-1}\|
\leq K\sum_{t=1}^T \frac{\|g_t\|_\star ^2}{Z_{t-1}}\|x_t-\overline{x}_t\|
\leq K D\sum_{t=1}^T \frac{\|g_t\|_\star ^2}{Z_{t-1}}~.
\]
Using Cauchy–Schwarz inequality, we have
\begin{align*}
M\sum_{t=1}^T \sqrt{Z_t}\frac{\|g_t\|_\star ^2}{Z_{t-1}}\|x_t-\overline{x}_t\|
&\leq M\sqrt{\sum_{t=1}^T \frac{\|\tilde{g}_t\|_\star ^2}{Z_{t-1}}} \sqrt{\sum_{t=1}^T \frac{Z_t}{Z_{t-1}}\|\tilde{g}_t\|_\star ^2\|x_t-\overline{x}_t\|^2 }~.
\end{align*}
So, putting together the last inequalities, we have
\[
\sum_{t=1}^T \langle \tilde{g}_t, \overline{x}_{t-1}-\overline{x}_T\rangle \leq M\sqrt{\sum_{t=1}^T \frac{\|\tilde{g}_t\|_\star ^2}{Z_{t-1}}} \sqrt{\sum_{t=1}^T \frac{Z_t}{Z_{t-1}}\|\tilde{g}_t\|_\star ^2\|x_t-\overline{x}_t\|^2 } + K D\sum_{t=1}^T \frac{\|g_t\|_\star ^2}{Z_{t-1}}~.
\]
We now focus on the the term $\sum_{t=1}^T \frac{\|g_t\|_\star ^2}{Z_{t-1}}$ that is easily bounded:
\begin{align*}
\sum_{t=1}^T \frac{\|g_t\|_\star ^2}{Z_{t-1}}
&= \sum_{t=1}^T \left(\frac{\|\tilde{g}_t\|_\star ^2}{Z_{t}} + \frac{\|\tilde{g}_t\|_\star ^2}{Z_{t-1}}-\frac{\|\tilde{g}_t\|_\star ^2}{Z_{t}}\right)\\
&\leq \sum_{t=1}^T \left(\frac{\|\tilde{g}_t\|_\star ^2}{Z_{t}} + \frac{1}{Z_{t-1}}-\frac{1}{Z_{t}}\right) \\
&\leq \frac{1}{Z_{0}}+\sum_{t=1}^T \frac{\|\tilde{g}_t\|_\star ^2}{Z_{t}}  \\
&\leq \frac{1}{Z_{0}}+\log\frac{Z_T}{Z_0} \\
&= 1+\ln Z_T,
\end{align*}
where in the last inequality we used the well-known inequality $\sum_{t=1}^T \frac{a_t}{a_0+\sum_{i=1}^t a_i} \leq \ln(1+\frac{\sum_{t=1}^T a_t}{a_0}),  \ \forall a_t\geq0$.

To upper bound the term $\sum_{t=1}^T \frac{Z_t}{Z_{t-1}}\|\tilde{g}_t\|_\star ^2\|x_t-\overline{x}_t\|^2$, observe that
\begin{align*}
\sigma_T^2 Z_T 
&=\|\overline{x}_0-\overline{x}_T\|^2+\sum_{t=1}^T \|\tilde{g}_t\|_\star ^2\|x_t-\overline{x}_T\|^2\\ 
&=\|\overline{x}_0-\overline{x}_T\|^2+\sum_{t=1}^{T-1} \|\tilde{g}_t\|_\star ^2\|x_t-\overline{x}_T\|^2 + \|\tilde{g}_T\|_\star ^2\|x_T-\overline{x}_T\|^2 \\
&= Z_{T-1}(\sigma_{T-1}^2+\|\overline{x}_T-\overline{x}_{T-1}\|^2) +\|\tilde{g}_T\|_\star ^2\|x_T-\overline{x}_T\|^2 \\
&= Z_{T-1}\sigma_{T-1}^2+\|\tilde{g}_T\|_\star ^2\left(1+\frac{\|\tilde{g}_T\|_\star ^2}{Z_{T-1}}\right)\|x_T-\overline{x}_T\|^2 \\
&= Z_{T-1}\sigma_{T-1}^2+\|\tilde{g}_T\|_\star ^2 \frac{Z_T}{Z_{T-1}} \|x_T-\overline{x}_T\|^2,
\end{align*}
where the third equality comes from bias-variance decomposition and the fourth one comes from~\eqref{eq:lemma_shortcut_eq1}. Hence, we have
\[
\sum_{t=1}^T \frac{Z_t}{Z_{t-1}} \|\tilde{g}_t\|_\star ^2\|x_t-\overline{x}_t\|^2 
= \sum_{t=1}^T (\sigma_t^2 Z_t - \sigma_{t-1}^2 Z_{t-1} ) \leq \sigma^2_T Z_T~.
\]

Putting all together, we have the stated bound.
\end{proof}
\end{document}